\documentclass[a4paper,fleqn]{cas-dc}

\usepackage[numbers]{natbib}
\usepackage{amsmath}
\usepackage{subfigure}
\usepackage{hyperref}

\DeclareMathOperator*{\argmin}{argmin} 
\usepackage{amsmath,bm}

\def\tsc#1{\csdef{#1}{\textsc{\lowercase{#1}}\xspace}}
\usepackage[ruled]{algorithm2e}
\newtheorem{definition}{Definition}

\newtheorem{theorem}{Theorem}

\tsc{WGM}
\tsc{QE}
\tsc{EP}
\tsc{PMS}
\tsc{BEC}
\tsc{DE}

\begin{document}
\let\WriteBookmarks\relax
\def\floatpagepagefraction{1}
\def\textpagefraction{.001}

\shorttitle{ }

\shortauthors{Xinlin~Leng.}

\title [mode = title]{ Federated Coordinate Descent for Privacy-Preserving Multiparty Linear Regression}

\author[1]{Xinlin~Leng}[style=chinese]

\author[2]{Chenxu~Li}
[style=chinese]
\author[1]{Weifeng~Xu}
[style=chinese]
\author[3]{Yuyan~Sun}
[style=chinese]

\author[1,4]{Hongtao~Wang}[style=chinese]

\cormark[1]

\address[1]{School of Control and Computer Engineering, North China Electric Power University, Baoding, 071051, China}
\address[2]{School of Computer Science and Technology, Shandong University of Finance and Economics, Jinan, 250014, China}
\address[3]{Beijing Key Laboratory of IoT Information Security, Institute of Information Engineering, Chinese Academy of Sciences, Beijing, 100093, China}
\address[4]{Hebei Key Laboratory of Knowledge Computing for Energy \& Power, Baoding, 071051, China}
\cortext[cor1]{Corresponding author}

\begin{abstract}
Federated linear regressions have been developed and applied in various domains, where multiparties collaboratively and securely perform optimization algorithms, e.g., Gradient Descent, to learn a set of optimal model weights.
    However, traditional Gradient-Descent based methods fail to solve problems which contain $\mathcal{L}_1$ regularization in objective functions, such as lasso regression.
    In this paper, we present Federated Coordinate Descent, a new distributed scheme called FCD, to address this issue securely under multiparty scenarios.
    Specifically, through secure aggregation and noise perturbations, our scheme guarantees that: (1) no local information is leaked to other parties, and (2) global model weights are not exposed to cloud servers.
   The proposed FCD scheme fills the gap between multiparty secure learning and coordinate descent methods, and is applicable for general linear regressions, including linear, ridge, and lasso regressions.
    Theoretical analyses prove the security guarantees of our scheme against semi-honest attackers.
    Experimental results on three UCI datasets demonstrate that regressions under our FCD scheme are effective as centralized models. 
    The performance evaluation on synthetic dataset also reveals that our scheme achieves efficient  properties on computation cost and communication overhead.

\end{abstract}


\maketitle

\section{Introduction}

Linear regression is one of the fundamental models in machine learning. 
It have been applied to a broad applications, such as decision support systems~\cite{es/TelesRKRA21}, time series forecasting~\cite{pr/IlicGCB21}, climate prediction~\cite{kdd/LiuCGD19}, smart grid~\cite{tsg/LiuZWYK19}, and signal processing~\cite{tsp/HellkvistOA21}, to name a few.
Traditional linear regression model usually performs in a centralized way, where data is gathered and processed on a central server.
However, centralized linear regression suffers from possible privacy leakage of sensitive data owned by distributed parties~\cite{icml/HammCB16}.

To tackle this issue, distributed linear regression approaches have recently been proposed, in which  multiple parties collaboratively learn a global model by sharing a subset of the parameters~\cite{tit/YuanPS21,tac/YangZYS21,jmlr/DobribanS20,aaai/HuWW20}. 
Although private local data is not directly exposed, sharing model  parameters may also leak sensitive information by insider attackers~\cite{ccs/HitajAP17,sp/MelisSCS19}.
In recent years, a new distributed machine learning paradigm was proposed to address such a problem, namely, federated learning.
It aims to provide privacy guarantees via security techniques such as secure multiparty computation~\cite{esorics/BogdanovLW08},  differential privacy~\cite{tamc/Dwork08}, and homomorphic encryption~\cite{rivest1978data}.
Thus, several privacy-preserving linear regression approaches~\cite{tdp/Dankar15,popets/GasconSB0DZE17,asiaccs/LuD21, nips/BernsteinS19} have been proposed under the federated learning framework.
As an optimization method, the Gradient Descent method is widely applied in multiparty linear regression~\cite{sp/MelisSCS19,isci/WangZLZL21}, and multiparty ridge regression with $\mathcal{L}_2$ regularization~\cite{isci/ChenRT18,isci/WangZLZL21}.

Unfortunately, the original gradient descent method fails to work with least-absolute shrinkage and selection operator (lasso) regression, whose optimization objective includes $\mathcal{L}_1$ regularization.
Lasso has demonstrated powerful performance on many tasks, e.g., genomics~\cite{d2009combining} and cancer research~\cite{kidd2018survival}, especially for selecting under-determined but sparse features.
To achieve multiparty lasso regression, ~\cite{icassp/BazerqueMG10} developed a distributed algorithm by reformulating the objective of the lasso into a separable form, which can be minimized by the Alternating-Direction Method of Multipliers (ADMM) method.
Zhao et al.~\cite{nips/ZhaoZLL18} proposed pSCOPE, a method for distributed sparse learning with $\mathcal{L}_1$ regularization.
However, these studies mainly focus on how to perform distributed regressions efficiently, and haven't provided sufficient security guarantees.
Zheng et al.~\cite{sp/ZhengPGS19} devised Helen, a system which uses partially homomorphic encryption and zero-knowledge proof to ensure secure ADMM implementation among multiple parties. 
Proximal Gradient Descent (PGD) based algorithm is also proposed for  secure multiparty lasso regression approach~\cite{midm/EgmondSGIVKSRLK21}, and is applied in two-party scenario.

Although GD, ADMM and PGD based multiparty linear regressions~\cite{sp/ZhengPGS19,midm/EgmondSGIVKSRLK21} have been proposed, a more general and effective optimization method, namely Coordinate Descent, has not been extensively studied under multiparty scenarios.
Coordinate descent~\cite{wu2008coordinate} is a very straightforward method in convex optimization that can be surprisingly efficient and scalable~\cite{cd2015}.
In this paper, we fill the gap to devise a federated coordinate descent approach, which aims to achieve privacy-preserving linear regression via collaboration between cloud servers and local data owners.
We hope regressions with different regularizations can be addressed, such as lasso regression.

However, it is a non-trivial task to develop a federated coordinate descent scheme for three challenges.
The first challenge is how to ensure that local private information wouldn't be leaked to other parties.
Second, since the the regression weights are vital and valuable assets belonging to local agents, they are reluctant to leak global model weights to any third parties except local data owners.
Thus, how to prevent cloud server(s) from stealing actual model weights in the learning process is another major challenge.
Third, we should ensure that the learned global models under multiparty scenarios won't be  affected by security mechanisms (e.g., encryption algorithms, noise perturbations, etc.). Namely, they can achieve good performance as accurately as centralized approaches. 

In this paper, we propose a privacy-preserving and secure multiparty scheme called Federated Coordinate Descent (FCD), to simultaneously estimate global models weights and preserve local data privacy.
To address the aforementioned challenges, our FCD scheme takes two main security mechanisms: 
(1) Data are pre-processed locally into some intermediate quantities, which are also encrypted by homomorphic plus operations and then shared to cloud servers for model training, without exposing sensitive local data to other parties; 
(2) Two noise perturbation techniques are devised to selected essential intermediate quantities for protecting the actual cloud-trained global weights from leakage to third parties, i.e., cloud evaluator and cryptographic service provider.
We guarantee that the noise perturbations can be eliminated when inaccurate weights returned to local data owners, and finally they can derive global model weights with high accuracy performance.
In addition, we successfully apply the FCD scheme to three regression models: linear regression, ridge regression, and lasso regression with $L_1$ regularization.
Finally, to evaluate the performance of our proposed FCD scheme in regression tasks, we take extensive studies in experiments on several UCI and synthetic datasets. 
The experimental results demonstrate that our FCD scheme has good performance, both in efficiency and effectiveness.
The main contributions achieved are as follows.
\begin{itemize}
    \item To the best of our knowledge, we are the first to propose a federated coordinate descent scheme for linear regression, ridge regression, and especially for lasso regression to address the underivative of $L_1$ regularization.
    
    \item We design specific secure aggregation and perturbation techniques to simultaneously solve security issues in the multiparty regressions for both local data and global model weights.

    \item We theoretically prove the security and usability of the proposed FCD scheme, i.e., it achieves high regression performance while preserving privacy from any other third parties.
    Experimental results show good performance of our scheme both on effectiveness and efficiency.
    
\end{itemize}

The rest of this paper is organized as follows. Section \ref{related} briefly introduces related work. Section \ref{pre} reviews the coordinate descent method along with linear, ridge, lasso regression, and Paillier encryption as preliminaries. Section \ref{overview} proposes the outline and key solution of the FCD scheme. Section \ref{detail} introduces the content of the FCD scheme in detail based on Section \ref{overview}, and carries out some  security analyses. Section \ref{experiment} reports the experimental results of the performance and security evaluation of the FCD scheme. Section \ref{conlusion} concludes the paper.

\section{Related Work}
\label{related}
In this section, we give a brief review of related work on: 
(1) distributed computing on optimization approaches;
and (2) privacy preservation for multiparty  linear regressions.

\subsection{Distributed Computing on Optimization Approaches}
Most machine learning models have formed an objective function to determine the best model parameters. 
To solve objective functions, several optimization approaches, including both convex and non-convex optimizations, have been adopted such as Gradient Descent~\cite{nips/BottouB07}, ADMM~\cite{mp/EcksteinB92}, Coordinate Descent~\cite{wu2008coordinate}, RMSprop~\cite{icml/MukkamalaH17}, Adagrad~\cite{icml/MukkamalaH17}, \\Adam~\cite{corr/KingmaB14}, Newton and quasi-Newton method~\cite{dennis1977quasi}, etc.
To run these optimizations on multiple machines or parties, existing distributed computing approaches to optimizations can be classified into two categories for the learning purpose: parallel learning and federated learning.

Parallel learning aims to perform high-throughput algorithm for optimization approaches on parallel computing clusters. 
The most critical research on parallel learning is to parallelize Gradient Descent~\cite{nips/ZinkevichWSL10,nips/TengW18,icml/XieKG19,aaai/GeorgeG20} for its widespread adoptions on machine learning tasks.
Based on these researches, many high-level machine learning models can be transferred to the distributed learning paradigms, for example, distributed ridge regression~\cite{icml/ShengD20}, distributed multi-task learning~\cite{aaai/YangL20a}, distributed tensor decomposition~\cite{www/HeHH19}, and distributed deep learning~\cite{icml/AssranLBR19}.
However, parallel learning assumes that data is centralized in one party, which has multiple computing devices.
Nowadays, in many areas, data is usually collected by multiparty, who are unwilling to share their data with a centralized party for privacy concerns.
Directly applying parallel learning to multiparties as computing clusters is intractable for the possible privacy leakage of sensitive data stored in each party. 

To address the privacy issues of multiparty machine learning, the concept of federated learning~\cite{tist/YangLCT19} is proposed to provide strong privacy guarantees via many security techniques, such as secure multiparty computation~\cite{esorics/BogdanovLW08}, differential privacy~\cite{tamc/Dwork08}, and homomorphic encryption~\cite{rivest1978data}.
In recent years, there is an extensive literature on federated learning with different optimization problems.
One of the popular studies is to solve the distributed privacy-preserving Gradient Descent~\cite{dais/DannerJ15,nss/Phong17,scn/DannerBHJ18}.
Jayaram et al.~\cite{IEEEcloud/JayaramVVTS20} designed a secure federated Gradient Descent technique to train machine learning or deep learning models collaboratively.
Wu et al.~\cite{tsp/WuLCG20} proposed a federated Stochastic Gradient Descent-based optimization for learning over multiparties under malicious Byzantine attack assumption.
Tan et al.~\cite{scn/TanZHG21} proposed distributed privacy-preserving Gradient Descent schemes on horizontally and vertically partitioned data.
ADMM is another kind of thriving optimization method, which has been studied for distributed and private learning in many literatures~\cite{ccs/ZhangZ16,tifs/ZhangZ17,icml/ZhangKL18,bigdataconf/DingZCXZP19,tifs/HuangHGCG20,tifs/ZhangKL20}. 

Coordinate descent is also a very straightforward and significant method in convex optimization.
To the best of our knowledge, most distributed Coordinate Descent approaches are designed for communication efficient parallel computing~\cite{cdc/Necoara12a,amcc/NecoaraF15,jmlr/RichtarikT16,jmlr/MahajanKS17,pakdd/ZhaoZZLC20}. 
Undoubtedly, there still exists a substantial lack of studying privacy-preserving Coordinate Descent in multiparty scenarios.
In this paper, we focus on federated multiparty Coordinate Descent, which addresses the possible privacy and security issues.

\subsection{Privacy Preservation for Multiparty Linear Regression}
Linear regression is a fundamental model in machine learning and has many applications. 
Developing distributed and privacy-preserving linear regression approaches has recently received significant attention.

Until now, many efforts on distributed linear regression are to combine multiple parties collaboratively to learn a global model by sharing a subset of the parameters~\cite{tit/YuanPS21,tac/YangZYS21,jmlr/DobribanS20,aaai/HuWW20}. 
However, sharing local parameters may also leak sensitive information~\cite{ccs/HitajAP17,sp/MelisSCS19}.
As a result, several privacy-preserving linear regression approaches have been proposed under the federated learning framework~\cite{tdp/Dankar15,popets/GasconSB0DZE17,asiaccs/LuD21, nips/BernsteinS19}.
To solve the optimization problem of linear regression, most of them are based on the federated Gradient Descent method~\cite{sp/MelisSCS19,isci/WangZLZL21,isci/ChenRT18}. 
Indeed, federated Gradient Descent is applicable for ordinary multiparty linear regression without any regularizations~\cite{sp/MelisSCS19,isci/WangZLZL21}, or multiparty ridge regression with $\mathcal{L}_2$ regularization~\cite{isci/ChenRT18,isci/WangZLZL21}.

However, the original Gradient Descent method is limited to Lasso regression, which has a crucial role in many domains such as cancer research~\cite{kidd2018survival}, Alzheimer's disease modeling~\cite{tkdd/LiuCGZB18}, and gene expression~\cite{cmmm/KanekoHH15}.
This is because the optimization objective function of lasso regression includes both $\mathcal{L}_1$ and $\mathcal{L}_2$ regularizations.
To achieve privacy-preserving multiparty lasso regression, Bazerque et al.~\cite{icassp/BazerqueMG10} developed a distributed algorithm by reformulating the objective of lasso into a separable form, which can be minimized by ADMM.
Zhao et al.~\cite{nips/ZhaoZLL18} proposed pSCOPE, a method for distributed sparse learning with $\mathcal{L}_1$ regularization.
However, these studies haven't provided sufficient security guarantees and mainly focus on distributed learning for efficiency.
Zheng et al.~\cite{sp/ZhengPGS19} devised Helen, a system that uses partially homomorphic encryption and zero-knowledge proof to guarantee secure ADMM implementation. 
But the computation cost of zero-knowledge proof is very high.
\\ Egmond et al.~\cite{midm/EgmondSGIVKSRLK21} proposed a secure multiparty Lasso regression approach with the Proximal Gradient Descent algorithm, but it is applied only in the two-party scenario.

Coordinate descent~\cite{wu2008coordinate} is a general method for solving lasso regression with $\mathcal{L}_1$ regularization.
But to the best our knowledge, it has not been explored extensively for privacy-preserving linear regression and lasso regression in multiparty scenarios.
Kesteren et al.~\cite{corr/abs-1911-03183} proposed a privacy-preserving protocol for generalized linear models using distributed Block Coordinate Descent, sharing a set of parameters.
However, it is designed only for vertically partitioned data and provides few security techniques for privacy guarantee.
In this paper, we draw on the theory of the Federated Coordinate Descent approach, in which data is horizontally partitioned and distributed among multiparties.
We aim to achieve privacy-preserving multiparty regression, including linear, ridge, and lasso regressions.

\section{Preliminary}
\label{pre}
In this section, we briefly illustrate the relevant preliminaries required for this paper, including the coordinate descent method, linear regression and paillier homomorphic encryption algorithm.
\subsection{Coordinate Descent}
Coordinate Descent is a classical optimization algorithm. The main idea of Coordinate Descent is to minimize only along one axis direction in each iteration, while the values in other axes are fixed, and thus making the multivariate optimization problem into a univariate optimization problem. 
Unlike the gradient descent method, it searches along a single-dimensional direction each time. 
After a number of iterations in which a minimum value in the current dimension is obtained, it finally converges to the optimal solution. Given the objective function $f(\textbf{w})$ where $\textbf{w}=(w_{0}, w_{1}, \ldots, $ $w_{n}) \in \mathbb{R}^{n+1}$ is a vector, the coordinate descent method is performed by
\begin{eqnarray}\label{equ:cd}
\begin{aligned}
&w_{0}^{(t)}=\argmin _{w_{0}}\ f\left(w_{0}, w_{1}^{(t-1)}, w_{2}^{(t-1)}, \ldots w_{n}^{(t-1)}\right), \\
&w_{1}^{(t)}=\argmin _{w_{1}}\ f\left(w_{0}^{(\mathrm{t})}, w_{1}, w_{2}^{(t-1)}, \ldots w_{n}^{(t-1)}\right), \\
&\ldots \ldots \\
&w_{n}^{(t)}=\argmin _{w_{n}}\ f\left(w_{0}^{(\mathrm{t})}, w_{1}^{(\mathrm{t})}, w_{2}^{(t)}, \ldots w_{n-1}^{(t)}, w_{n}\right),
\end{aligned}
\end{eqnarray}
where $t$ indicates the current iteration.
Note that the search of the minimum value in each dimension of $\textbf{w}$ during the iterations $(t)$ is conducted by the results of the previous iteration $(t-1)$ until  convergence.

\subsection{Linear Regression}
Given a dataset $D=\{(\textbf{x}_1,y_1),(\textbf{x}_2,y_2),\cdots,(\textbf{x}_m,y_i)\}$ with $m$ samples, where $\textbf{x}_i=(x_{i1}, x_{i2}, \ldots, x_{in})$ is a $n$-dimensional vector and $y_i$ is the corresponding target value, 
the aim of linear regression is to learn a vector of optimal regression coefficients $\textbf{w}=\{w_0,w_1,\cdots,w_n\}$ to fit the dataset $D$. 
The hypothesis function of linear regression model is $h_{\textbf{w}}(\textbf{x}_i)=\textbf{w}\cdot \textbf{x}_i$, where we rewrite the input vector as $\textbf{x}_i=(x_{i0},x_{i1}, x_{i2}, \ldots,$ $ x_{in})$ to match the dimension of $\textbf{w}$ since $x_{i0}=1$.
The sum of squared error is used as the cost function to obtain the optimal solution, i.e.,
$$f(\textbf{w})=\sum_{i=1}^{m}\left(y_{i}-\textbf{w}\cdot \textbf{x}_i\right)^{2}=\sum_{i=1}^{m}\left(y_i-\sum_{j=0}^{n} x_{ij} w_{j}\right)^{2}
$$.

Instead of gradient descent to optimize the cost function, we adopt the coordinate descent method to find the optimal $\textbf{w}$ by \eqref{equ:cd}.
To solve the $\argmin$ objective function,  
the partial derivative of $f(\textbf{w})$ over $w_k$ is computed by
$$
\frac{\partial f(\textbf{w})}{\partial w_{k}}=-2 P_{k}+2 Z_{k} w_{k}, \quad k=0,\cdots, n,
$$
where 
\begin{equation}\label{equ:pk}
P_{k}=\sum_{i=1}^{m} x_{i k}\left(y_{i}-\sum_{j=0, j \neq k}^{n} x_{i j} w_{j}\right),
\end{equation}
and
\begin{equation}\label{equ:Zk}
Z_{k}=\sum_{i=1}^{m} x_{i k}^{2}.
\end{equation}
Let $\frac{\partial f(\textbf{w})}{\partial w_{k}}=0$, and we get 
\begin{equation}\label{equ:wk1}
w^\ast_{k}=\argmin_{w_k} f(\textbf{w})=P_{k} / Z_{k}, \quad k=0,\cdots, n.
\end{equation}

By running \eqref{equ:cd} using solution \eqref{equ:wk1} iteratively until convergence, we finally derive the optimal regression weights.

\subsection{Ridge Regression}
Ridge regression is one of the most significant transformations of linear regression.
Compared to linear regression, ridge regression regularizes the weights using a quadratic penalty to avoid the problem of possible overfitting, and thus improves prediction accuracy.
By introducing an $L_{2}$ regularization term, the cost function of ridge regression is
$$
f(\textbf{w})=\sum_{i=1}^{m}\left(y_{i}-\textbf{w}\cdot \textbf{x}_i\right)^{2}+\lambda \sum_{j=0}^{n} w_{j}^{2},
$$
where $\lambda$ is a parameter controlling the regularization effect.
To perform coordinate descent by \eqref{equ:cd}, we compute the partial derivatives as
$$
\frac{\partial f(\textbf{w})}{\partial w_{k}}=-2 P_{k}+2 Z_{k} w_{k}+2\lambda w_{k}
$$
where $P_k$ and $Z_k$ are defined by \eqref{equ:pk} and \eqref{equ:Zk}.
Let $\frac{\partial f(\textbf{w})}{\partial w_{k}}=0$, and we get
\begin{equation}\label{equ:wk2}
w^\ast_{k}=\frac{P_{k}}{Z_{k}+\lambda}, \quad k=0,\cdots, n.
\end{equation}
The optimal regression weights can also be solved by the coordinate descent method combining \eqref{equ:cd} and \eqref{equ:wk2}.

\subsection{Lasso Regression} 
Lasso regression is a popular model for parameter learning and variable selection in linear regression problems.
It is well-suited for the sparse scenario where most weights in $\textbf{w}$ shrink to zero.
Different from ridge regression, lasso uses the penalty of $L_{1}$ regularization, which is defined as the sum of absolute values of weights. 
The cost function of the lasso is formalized as
$$
f(\textbf{w})=\sum_{i=1}^{m}\left(y_{i}-\textbf{w}\cdot \textbf{x}_i\right)^{2}+\lambda \sum_{j=0}^{n}\left|w_{j}\right|.
$$
Since the $L_1$ regularization term of $f(\textbf{w})$ is non-differentiable, we use the sub-gradient instead of the gradient to get
$$
\begin{aligned}
\frac{\partial f(\textbf{w})}{\partial w_{k}} &=\left\{\begin{array}{cc}
2 Z_{k} w_{k}-2 P_{k}-\lambda & w_{k}<0 \\
{\left[-2 P_{k}-\lambda,-2 P_{k}+\lambda\right]} & w_{k}=0 \\
2 Z_{k} w_{k}-2 P_{k}+\lambda & w_{k}>0
\end{array}\right..
\end{aligned}
$$
Let $\frac{\partial f(\textbf{w})}{\partial w_{k}}=0$, and we get
\begin{equation}\label{equ:wk3}
	\begin{aligned}
w^\ast_{k}= \begin{cases}\left(P_{k}+\lambda / 2\right) / Z_{k} & P_{k}<-\lambda / 2 \\ 0 & -\lambda / 2 \leq P_{k} \leq \lambda / 2 \\ \left(P_{k}-\lambda / 2\right) / Z_{k} & P_{k}>\lambda / 2\end{cases}
	\end{aligned}
\end{equation}
Finally, the optimal regression weights are obtained by performing the coordinate descent method via \eqref{equ:cd} and \eqref{equ:wk3}.

\subsection{Paillier Encryption}
\label{paillier}
Paillier encryption is a probabilistic public key encryption algorithm proposed by Pascal Paillier~\cite{eurocrypt/Paillier99} in 1999, which is based on the difficult problem of compound residual class.
It is a widespread realization of homomorphic encryption satisfying homomorphic plus operation and homomorphic scalar multiplication, and has been widely used in cryptographic signal processing and federated learning.
The main functions of the Paillier cryptosystem are as follows.

\noindent\textbf{Key Generation:}

(1) First, select two large prime numbers $p$ and $q$ of the same length randomly, and make them satisfy $\operatorname{gcd}(p q,(p-1)(q-1))=1$, where $\operatorname{gcd}()$ means the greatest common divisor.

(2) Compute $N=p q$ and $\lambda=\operatorname{lcm}(p-1, q-1)$, where $\operatorname{lcm}()$ denotes the minimum common multiple and $|N|$ denotes the bit length of $N$.

(3) Randomly selects an integer $g$ from the integer set $\mathbb{Z}_{N^{2}}^{*}$: $g \leftarrow \mathbb{Z}_{N^{2}}^{*}$.

(4) Define the function $L$: $L(x)=\frac{x-1}{N}$, simultaneously calculate $\mu=\left(L\left(g^{\lambda} \bmod N^{2}\right)\right)^{-1} \bmod N$, and define the public key as $(N, g)$ and private key as $(\lambda, \mu)$.

\noindent\textbf{Homomorphic plus:} For ciphertext $c_{1}=E[m_1]$ and $c_{2}=E[m_2]$, we define the ciphertext homomorphic plus $'\oplus'$ operation as $c_{1} \oplus c_{2} =E[m_1]\oplus E[m_2]=E[m_1+m_2]$.

\noindent\textbf{Homomorphic scalar multiplication:} For ciphertext $c_{1}=E[m_1]$ and scalar $\alpha$, we define the ciphertext homomorphic multiplication $'\otimes'$ operation as $\alpha \otimes c_1=\alpha \otimes E\left(m_1\right)=E\left(\alpha \cdot m_1\right)$.

\section{Scheme Overview}
\label{overview}
In this section, we first introduce the architecture and design goals of our multiparty linear regression system.
Then, we present the threat model, the key ideas, and the security goals of our FCD scheme.

\subsection{Multiparty Linear Regression System}
\label{sec:system}
In this paper, we build a multiparty linear regression system.
Given a dataset $D=\{(\textbf{x}_1,y_1),(\textbf{x}_2,y_2),\cdots,(\textbf{x}_m,y_m)\}$ with $m$ samples, where $\textbf{x}_i=(x_{i1}, x_{i2}, \ldots, x_{in})\in \mathbb{R}^n$  and $y_i$ is the corresponding target value. Suppose the samples are distributed in $M$ parties and the number of samples in party $l$ is $m_l$.
Thus, we get $\sum_{l=1}^{M}\sum_{i=1}^{m_l} = m$.
The goal of these parties is to collaboratively learn a linear regression model $h_{\textbf{w}}(\textbf{x}_i)=\textbf{w}\cdot \textbf{x}_i$, such that the samples in $D$ are best fitted.

To meet the privacy requirements, 
our system mainly consists of three components: local data owner (\textit{DO}), cloud server for global training and evaluation (\textit{Evaluator}), and cryptographic service provider (\textit{CSP}), as shown in Fig. \ref{fig:arch}.

\begin{figure}[t]
\centering
\includegraphics[width=0.5\textwidth]{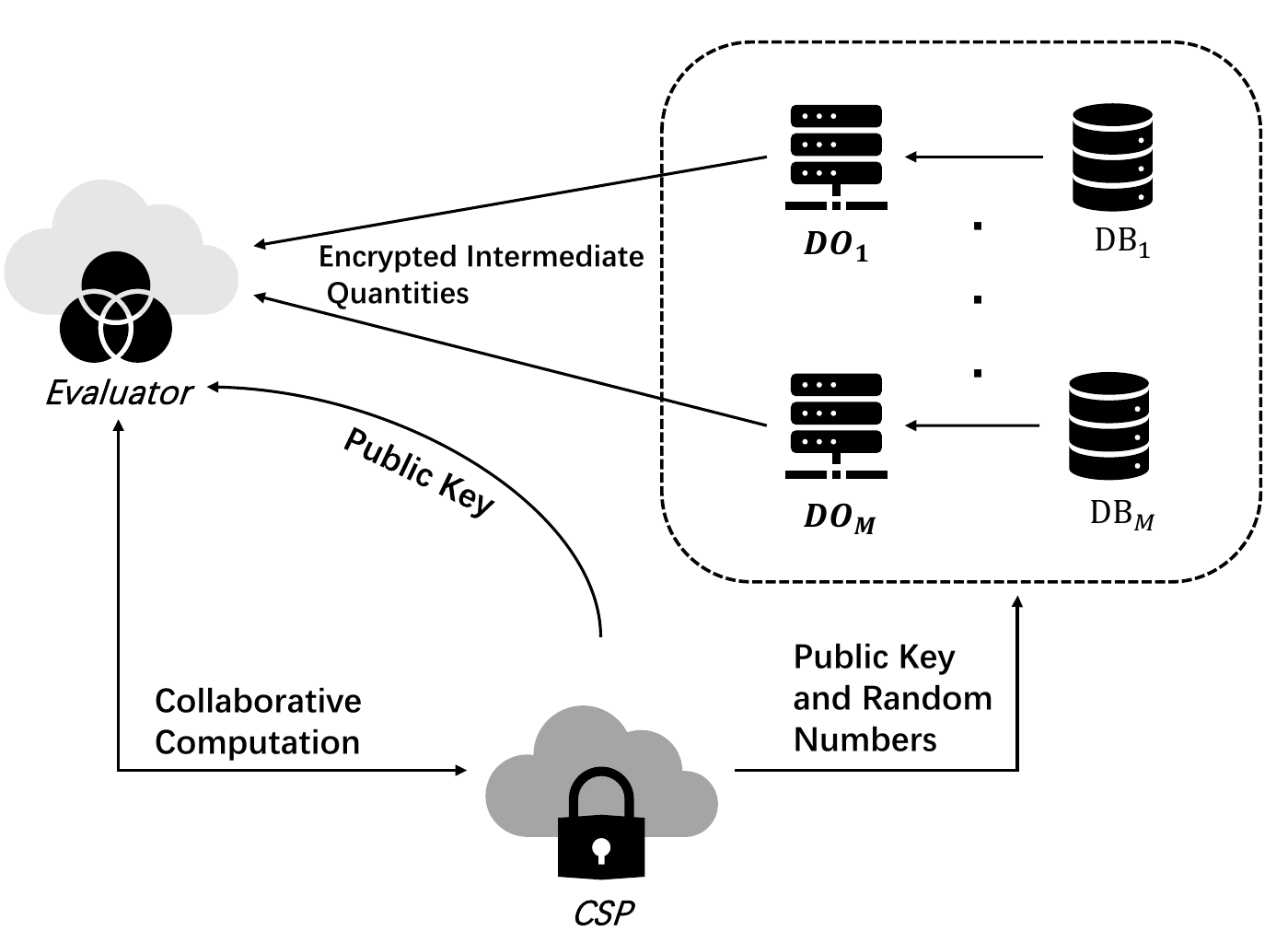}
\caption{Architecture of Multiparty Linear Regression System}
	\label{fig:arch}
\end{figure}

\textit{DO} is the owner of the local data, and we use ${DO}_l$ to denote the $l$-th local data owner $(l=1 \cdots M)$.
It is responsible for ${DO}_l$ to load and pre-process data from the local database (${DB}_l$).
To prevent information leakage, ${DO}_l$ encrypts the some intermediate quantities with the public key obtained from the \textit{CSP}, and sends them to the \textit{Evaluator} for global training.

\textit{CSP} is a cryptographic service provider that generates the key and the random number at system startup.
As shown in Fig. \ref{fig:arch}, the \textit{CSP} keeps the generated private keys and distributes the public keys or random numbers to the \textit{Evaluator} and each \textit{DO}.
During the aggregation and evaluation phase, it is responsive for the \textit{CSP} to decrypt ciphertext and do some auxiliary computation operations.

\textit{Evaluator} is the cloud server for global  training based on all encrypted intermediate quantities. 
To perform the coordinate descent algorithm on them, \textit{Evaluator} combines the  \textit{CSP} to finish the training process collaboratively.

The goals of our system are shown as follows.
\begin{enumerate}
    \item \textit{DO}, \textit{CSP}, and \textit{Evaluator} work together, and eventually, all \textit{DO}s obtain the same global regression model weights with high performance.
    \item \textit{DO}'s data and sensitive local information would not leak to other parties, i.e., \textit{CSP} and \textit{Evaluator}.
    \item Neither \textit{CSP} nor \textit{Evaluator} can identify accurate weights of the global regression model, because the weights are also private for \textit{DO}s.
\end{enumerate}

\subsection{Threat Model}
In our multiparty linear regression system, we consider \textit{DO}, \textit{CSP}, and \textit{Evaluator} to be semi-honest-but-curious.
Specifically, \textit{DO} honestly performs local computation and encryption, but it is curious about other \textit{DO}s' local data privacy.
\textit{CSP} and \textit{Evaluator} perform calculations and interactions strictly according to our training  scheme, but they are also curious about the local data information of \textit{DO}s and trying to infer any information about model parameters. 

We also assume that there is no collusion between \textit{Evaluator} and \textit{CSP}.
This is because once \textit{CSP} colludes with \textit{Evaluator}, the private key and sensitive information of the global regression weights would be leaked. 

\subsection{Key Idea of Our FCD Scheme}
\label{main idea}

We aim to perform the coordinate descent method (see Eq. \eqref{equ:cd}) in a federated way. 
To achieve the security goal under our threat model, the basic principle of our scheme is to send encrypts of some intermediate quantities from each \textit{DO}. 
Then, the \textit{Evaluator} can get the specific aggregated intermediate quantities with the help of \textit{CSP}, and continue to finish training on these quantities.

From Eq. \eqref{equ:wk1}, \eqref{equ:wk2}, \eqref{equ:wk3}, we can see that the key to perform the coordinate descent method for \textit{Evaluator} is to compute $P_k$ and $Z_k$ in a secure manner, because these two quantities need to collect data from each \textit{DO}.
We transform $P_k$ from Eq. \eqref{equ:pk} to the following: 
\begin{equation}\label{equ:pk1}
	\begin{aligned}
P_{k}=Q_{k}-\sum_{j=0}^{n} S_{k j} w_{j},
\end{aligned}
\end{equation}
where $Q_k=\sum_{l=1}^{M} q^{(l)}_k$ and $S_{kj}=\sum_{l=1}^{M} s^{(l)}_{kj}$ are aggregation of $q^{(l)}_k$ and $s^{(l)}_{kj}$ from all \textit{DO}s. 
For the $l$-th \textit{DO}, $q^{(l)}_k$ and $s^{(l)}_{kj}$ are defined as follows:

\begin{equation}\label{equ:qk}
\begin{aligned}
q^{(l)}_{k}=\sum_{i=1}^{m_l} x_{i k} y_{i},
\end{aligned}
\end{equation}
and
\begin{equation}\label{equ:skj}
\begin{aligned}
&s^{(l)}_{k j}=\left\{\begin{array}{cl}
0 & k=j \\
\sum_{i=1}^{m_l} x_{i j} x_{i k} & k \neq j.
\end{array}\right.
\end{aligned}
\end{equation}
Similarly, for $Z_{k}$ in Eq. \eqref{equ:Zk}, we can also transform it into $Z_k=\sum_{l=1}^{M} z^{(l)}_k$ by the definition of 
\begin{equation}\label{equ:zk}
\begin{aligned}
z^{(l)}_{k}=\sum_{i=1}^{m_l} x^2_{i k}.
\end{aligned}
\end{equation}

Therefore, we can clearly see that  \textit{DO} does not need to upload local sensitive data for global model training, but computes $q^{(l)}_{k}$, $s^{(l)}_{kj}$ and $z^{(l)}_{k}$ locally.
To enhance the strength of data security, 
we use homomorphic encryption on these quantities and then upload them to the cloud \textit{Evaluator}. 
Then, the \textit{Evaluator} performs  homomorphic plus operation under the ciphertext, and obtains the global aggregation quantities $P_k$ and $Z_k$ from \textit{CSP} via decryption operation.
Finally, the \textit{Evaluator} runs coordinate descent algorithm to learn the global weights $\bm{w}$.
In this process, due to the mechanism of homomorphic encryption, the \textit{Evaluator} only gets a set of  encrypted and aggregated intermediate quantities, which cannot be inferred any sensitive information about the local data.

\begin{figure*}[t]
	\centering
	\includegraphics[width=\textwidth]{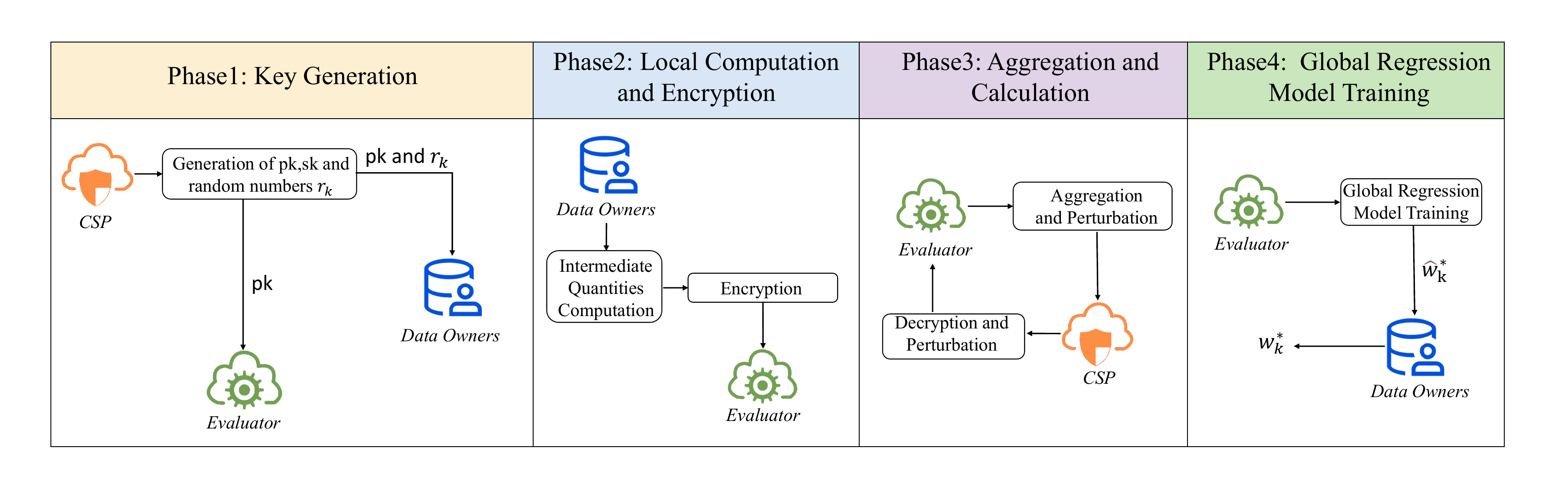}
	\caption{An overview of FCD Schemes.
	After the key generation phase, \textit{DO} performs local computation and encryption, and then sends encrypted intermediate quantities to the \textit{Evaluator} for aggregation.
	The \textit{Evaluator} completes the aggregations and sends them to the \textit{CSP} for decryption, and then returns the intermediate quantities to the \textit{Evaluator}, which performs the final regression model training.
	Eventually, only \textit{DO} gets the actual regression model weights.}
	\label{fig:phases}
\end{figure*}

\subsection{Protection of the Global Weights}
\label{security protect}
While the local private data is protected  via homomorphic encryption and aggregation, it is also essential to protect the global model weights that may be learned by \textit{Evaluator} and \textit{CSP}.

Therefore, we consider adding two perturbations in the training process of our FCD scheme.
Firstly, in each iteration $t$ of coordinate descent, we use a small noise $r_k$ to ensure that the weight $w_k$ will not be leaked to the \textit{Evaluator}. Our objective is:
\begin{equation}\label{equ:sc9}
	\begin{aligned}
    \widehat{w}_{k}=w_{k}+r_{k}, \qquad k=0,\cdots,n,
    \end{aligned}
\end{equation}
where $\widehat{w}_{k}$ is the weight the \textit{Evaluator} can derive.
Note that $r_k$ is added on an intermediate term when \textit{CSP} returns decrypted quantities to the \textit{Evaluator}.
We demonstrate the detailed information in \ref{cloud decrypt}.
We guarantee that during the aggregation and training phases,  \textit{Evaluator} can only derive $\widehat{w}_{k}$, which is not the actual model weight.
But \textit{DO}s can get the actual weight by $w_{k}=\widehat{w}_{k}-r_{k}$.

Second, when \textit{CSP} decrypts the intermediate quantities such as $P_k$ and $Z_k$, it can also infer the actual model weights by performing the coordinate descent method via \eqref{equ:cd} and \eqref{equ:wk3}.
To prevent weights leakage on \textit{CSP}, \textit{Evaluator} generates a set of random numbers ${\xi}_{k j}, k,j=0,\cdots n$. Then, it executes the paillier homomorphic scalar multiplication algorithm on ${\xi}_{k j}$ and the aggregated encryption $E[{S}_{k j}]$.
Since \textit{CSP} couldn't get the actual intermediate quantities, it would fail to infer the actual model weights. 
The detailed information is introduced in \ref{Addition noise in Eva} and \ref{model training}.

\section{Scheme Details}
\label{detail}
 In this section, we briefly illustrate the main phases of our FCD scheme. 
 Next, we describe the details of each phase and provide proof for the correctness of the scheme.

\subsection{Phases of the FCD Scheme}
The training process of the proposed FCD scheme is shown in Fig. \ref{fig:phases}. 
Our scheme consists of four phases: key generation, local computation and encryption, aggregation and calculation, and regression model training.
We give brief introductions as follows.

\begin{enumerate}
\item \textbf{Key Generation Phase}: \textit{CSP} generates keys and a set of random numbers $r_k (k=0,\cdots,n)$. 
Then, the public key is distributed to the \textit{Evaluator} and all \textit{DO}s. 
The random numbers are also distributed to each \textit{DO}.

\item \textbf{Local Computation and Encryption Phase}: Each \textit{DO} executes the computation and encryption algorithm on the local private data, and then sends the encrypted intermediate quantities to the \textit{Evaluator}.

\item \textbf{Aggregation and Calculation Phase}: The \textit{Evaluator} aggregates the received intermediate quantities under 
ciphertext from each \textit{DO}. Before sending them to the \textit{CSP} for decryption, the \textit{Evaluator} adds noises into the aggregations to avoid weights inference on the \textit{CSP}.
Then it sends perturbated aggregations to the \textit{CSP}.
Finally, the \textit{CSP} decrypts the aggregations and sends the results back to the \textit{Evaluator}.

\item \textbf{Global Regression Model Training Phase}: The \textit{Evaluator} uses the decrypted aggregations sent back from the \textit{CSP} to train the global regression model, and distributes the model weights with noise back to all \textit{DO}s.
Finally, \textit{DO} gets the final global model weights via a simple noise elimination.

\end{enumerate}

 \subsection{Key Generation Phase}
 Initially, \textit{CSP} executes the key generation algorithm as: $\text {KeyGen}(K) \rightarrow \text {( pk , sk )}$, where $K$ is the key length, pk and sk are the generated public key and private key,  respectively.
 Then \textit{CSP} distributes pk to the \textit{Evaluator} and all \textit{DO}s. 
 According to the number of features, the \textit{CSP} generates $n+1$ random numbers to form a vector $\bm{r}=(r_0,r_1,\cdots,r_n)$, and sends $\bm{r}$ to all \textit{DO}s.

\subsection{Local Computation and Encryption Phase}
\subsubsection{Intermediate Quantities Computation}
According to Eq. \eqref{equ:qk}, \eqref{equ:skj}, \eqref{equ:zk}, for the $l$-th \textit{DO} , $q^{(l)}_{k}$, $s^{(l)}_{kj}$, $z^{(l)}_k$ are the quantities to be locally calculated.
Besides those quantities, another critical quantity needs to be calculated before encryption.

To protect the global model weights, the initial weight $w_k$ is added by the noise number $r_k$ to be a perturbated weight $ \widehat{w}_k$ as shown in Eq. \eqref{equ:sc9}.
Substitute $ \widehat{w}_k$ into Eq. \eqref{equ:pk1}, we can get

\begin{equation}\label{equ:pk2}   
P_k=Q_{k}-\sum_{j=0}^{n} S_{k j} w_{j}-\Delta R_k, 
\end{equation}
where 
$$
	\Delta R_k = \sum_{j=0}^{n} S_{k j} r_j = \sum_{j=0}^{n} \sum_{l=1}^{M} s^{(l)}_{kj} r_j =  \sum_{l=1}^{M} \sum_{j=0}^{n} s^{(l)}_{kj} r_j.
$$
We denote 
\begin{equation}\label{delta r} 
	{\Delta r}^{(l)}_k=\sum_{j=0}^{n} s^{(l)}_{kj} r_j,  
\end{equation}
and then $\Delta R_k = \sum_{l=1}^{M} {\Delta r}^{(l)}_k$, which can be factorized and calculated by each \textit{DO}.

Hence ${\Delta r}^{(l)}_k\ (k=0,\cdots,n)$ are also quantities to be calculated for each \textit{DO}. 
Note that if the \textit{Evaluator} uses $P_k$ in Eq .\eqref{equ:pk2} to perform the coordinate descent algorithm in \eqref{equ:cd}, it is straightforward to see that the \textit{Evaluator} will finally derive $\widehat{w}_k$ instead of $w_k$.
This is one of our objectives to protect model weights.

\subsubsection{Encryption of Intermediate Quantities}
After finishing the computation of ${q^{(l)}_k}$, ${s^{(l)}_{k j}}$, ${z^{(l)}_k}$, ${{\Delta r}^{(l)}_k}$, \textit{DO} executes paillier encryption on them using the public key pk. For each $k,j = 0,\cdots,n$, we have
\begin{equation}
\begin{aligned}
E[q^{(l)}_k] & \Leftarrow \text { Encrypt }\left(\text { pk, } q^{(l)}_k\right), \\
E[s^{(l)}_{k j}] & \Leftarrow \text { Encrypt }\left(\text { pk, } s^{(l)}_{k j}\right), \\
E[z^{(l)}_{k}] & \Leftarrow \text { Encrypt }\left(\text { pk, } z^{(l)}_k\right), \\
E[\Delta r^{(l)}_{k}] & \Leftarrow \text { Encrypt }\left(\text { pk, } \Delta r^{(l)}_{k}\right).
\end{aligned}
\end{equation}
Then, these encrypted quantities as well as $\widehat{w}_{k}$ are sent together to the \textit{Evaluator} for data aggregation and global model training.

\subsection{Aggregation and Calculation Phase}
\subsubsection{Data Aggregation in the \textit{Evaluator}}
After receiving the encrypted quantities from all \textit{DO}s, the \textit{Evaluator} executes  aggregation using the additive homomorphic operation `$\oplus$' of the paillier algorithm as:

\begin{equation}\label{equ:agg}
\begin{aligned}
E[{Q}_{k}]& = E[{q}^{(1)}_k] \oplus  \cdots \oplus E[{q}^{(M)}_k], \\
E[S_{k j}] & =E[s^{(l)}_{k j}]\oplus \cdots \oplus E[s^{(M)}_{k j}],\\
E[Z_{k}] & =E[z^{(l)}_{k}] \oplus \cdots \oplus E[z^{(M)}_{k}], \\
E[\Delta R_{k}] & =E[\Delta r^{(l)}_{k}]\oplus \cdots \oplus  E[\Delta r^{(M)}_{k}].
\end{aligned}
\end{equation}
To this end, the \textit{Evaluator} derives the encrypted and aggregated intermediate quantities, which help run the coordinate descent algorithm.
Next, the \textit{Evaluator} will send them to the \textit{CSP} for decryption.

\subsubsection{Perturbation in the \textit{Evaluator}}
\label{Addition noise in Eva}
However, directly transmitting the aggregated intermediate quantities will lead to model weight leakage in the \textit{CSP}.
\begin{definition}[Weight Stolen Attack 1]
	If $E[Q_k]$, $E[S_{k j}]$, $E[Z_{k}]$ and $E[\Delta R_{k}]$ are sent to the \textit{CSP},  
	it can calculate the actual value of $P_k$ in Eq. \eqref{equ:pk1}.
	Therefore, after generating a random set of initial weights $w_k, (k=0,\cdots,n)$, the \textit{CSP} can obtain the final model weights by executing the coordinate descent algorithm individually.	
\end{definition}

To prevent the model stolen attack in the \textit{CSP} , we devise to add perturbations on one of the encrypted and aggregated intermediate quantities, i.e., $E[S_{kj}]$.
The \textit{Evaluator} firstly generates a set of random number ${\xi}_{k j}, (k,j=0,\cdots,n)$, and then executes the paillier homomorphic scalar multiplication operations `$\otimes $' on ${\xi}_{k j}$ and $E[{S}_{k j}]$ as
\begin{equation}\label{equ:sc17}
E[S^{\prime}_{k j}]=E[S_{k j}] \otimes   \xi_{k j},\  \left(k,j = 0,\cdots,n\right), 
\end{equation}
where $S^{\prime}_{kj}$ is the noisy form of $S_{kj}$.
After completing the perturbations by \textit{Evaluator}, it sends $E[Q_{k}]$, $E[{S^{\prime}}_{k j}]$, $E[Z_{k}]$, $E[{\Delta R}_{k}]$ to the \textit{CSP} for decryption.

\subsubsection{Decryption and Perturbation in the \textit{CSP}}
\label{cloud decrypt}
When receiving all the aggregated encrypted quantities from the \textit{Evaluator}, the \textit{CSP} performs the paillier decryption algorithm via private key as:
$$
\begin{aligned}
Q_{k} & \Leftarrow \operatorname{Decrypt}\left(\text {sk, } E[Q_{k}]\right) ,\\
S_{k j}^{\prime} & \Leftarrow \operatorname{Decrypt}\left(\text { sk , } E[S^{\prime}_{k j}]\right) ,\\
Z_{k} & \Leftarrow \operatorname{Decrypt}\left(\text { sk, } E[Z_{k}]\right) ,\\
\Delta R_{k} & \Leftarrow \operatorname{Decrypt}\left(\text { sk, } E[\Delta R_{k}]\right).
\end{aligned}
$$
Note that the \textit{CSP} couldn't infer the actual model weights because of $S^{\prime}_{kj}$, and we give a proof in subsection~\ref{protection}.

After decryption, the \textit{CSP} cannot send these intermediate quantities back to the \textit{Evaluator} directly.
The reason is the same as mentioned in subsection \ref{Addition noise in Eva}, and we give the following definition.
\begin{definition}[Weight Stolen Attack 2]
	If the \textit{Evaluator} receives all the decrypted intermediate quantities, i.e., $Q_{k}$, ${S^{\prime}}_{k j}$, $Z_{k}$, ${\Delta R}_{k}$, it can also infer the global model weights.
\end{definition}

To prevent the model weight leakage, we devise to add perturbations to $Q_k$ and $\Delta R_k$ in the \textit{CSP} before sending them back to the \textit{Evaluator}.
For each $k,j = 0,\cdots,n$, we set
\begin{equation}\label{equ:qkrk}
	\begin{aligned}
		Q_{k}^{\prime}&=Q_{k}+2\ \Delta R_{k},\\
		\Delta R_{k}^{\prime}&=\Delta R_{k} - Z_kr_k.
	\end{aligned}
\end{equation}
Note that the noise $r_k$ is the same as the \textit{CSP} distributed to each \textit{DO}, 
and $r_k$ is unknown to the \textit{Evaluator}.

Finally, the \textit{CSP} sends the decrypted ${S^{\prime}}_{k j}$, ${Z}_k$, and the perturbated ${Q^{\prime}}_k$, ${\Delta R^{\prime}}_k$ to the \textit{Evaluator} for training the global regression model.
Let $Dec.Pai$ denote the decryption function. The decryption and perturbation processes are illustrated in algorithm~\ref{alg:dec}.

\begin{algorithm}[t]
	\label{alg:dec}
\caption{Decryption and perturbation in \textit{CSP}}  
\LinesNumbered  
\KwIn{ $E[{Q}_k]$, $E[S^{\prime}_{k j}]$, $E[{Z}_k]$, $E[{\Delta R}_k]$, ${r}_k$, sk }
\KwOut{ $Q^{\prime}_k$ , $S^{\prime}_{k j}$ , ${Z}_k$ , ${\Delta R}^{\prime}_k$ }  
\For{$k = 0,\cdots, n$}{
$Z_k \gets Dec.Pai(E[Z_k],\text{sk})$\newline
$Q_k \gets Dec.Pai(E[Q_k],\text{sk})$\newline
$\Delta R_k \gets Dec.Pai(E[\Delta R_k],\text{sk})$\newline
$Q^{\prime}_k \gets Q_k + 2{\Delta R}_k$\newline
$\Delta R^{\prime}_k \gets {\Delta R}_k-{Z_k}r_k$\newline
\For{$j = 0,\cdots, n$}{
      $S^{\prime}_{k j} \gets Dec.Pai(E[S^{\prime}_{k j}],\text{sk})$
      }
    }
\end{algorithm}  

\subsection{Global Regression Model Training Phase}
\label{model training}
After receiving the decrypted data $Q^{\prime}_k$, $S^{\prime}_{k j}$, ${Z}_k$, and $\Delta R^{\prime}_k$ from the \textit{CSP}, the \textit{Evaluator} firstly eliminates the added noise to derive $S_{kj}$ according to Eq.~\eqref{equ:sc17}.
Then, by using  ${Q^{\prime}}_k$, $S_{kj}$, ${Z}_k$, and $\Delta R^{\prime}_k$, the \textit{Evaluator} performs the training process in \eqref{equ:cd}.
In each step, the optimal weight can be computed by \eqref{equ:wk1}, \eqref{equ:wk2}, and \eqref{equ:wk3} for linear, ridge, and lasso regression, respectively.

The \textit{Evaluator} can finally derive $\widehat{w}_k = w_k + r_k$, since $Q^{\prime}_k$ and ${\Delta R^{\prime}}_k$ are perturbated by the \textit{CSP}.
We give the following theorem:

\begin{theorem}\label{the1}
	Given $Q^{\prime}_k$, $S_{k j}$, ${Z}_k$, $\Delta R^{\prime}_k$ and initial weight $\widehat{w}_k$, supoose $w^\ast_k$ is the optimal true weight, the \textit{Evaluator} can derive optimal model weights $\widehat{w}^\ast_k = w^\ast_k + r_k$ by coordinate decent method.
\end{theorem}
\begin{proof}
Let $P^{\prime}_k$ be the perturbated $P_k$. According to Eq.\eqref{equ:pk2} and \eqref{equ:qkrk}, for each $k = 0,\cdots,n$, we have 
$$
\begin{aligned}
P^{\prime}_{k} &= Q_{k}^{\prime}-\sum_{j=0}^{n} S_{k j} \widehat{w}_{j} - \Delta R^{\prime}_{k}\\
&=Q_k + 2\ \Delta R_{k} - \sum_{j=0}^{n} S_{k j} (w_j+r_j) - (\Delta R_{k} - Z_kr_k)\\
&=P_{k} + Z_kr_k.
\end{aligned}
$$
Taking lasso regression as example, the optimal results in each step is illustrated by Eq. \eqref{equ:wk3}, and thus we have:
$$
	\begin{aligned}
		\widehat{w}^\ast_{k}&=\frac{P^{\prime}_{k} \pm \lambda / 2}{Z_{k}} 
		=\frac{P_{k} + Z_kr_k \pm \lambda / 2}{Z_{k}} 
	=w^\ast_{k}+r_{k}
	\end{aligned}
$$
\end{proof}

After a number of iterations when the solution of the coordinate descent method converges, the \textit{Evaluator} sends the final optimal model weight $\widehat{w}^\ast_k$ to all \textit{DO}s.
Then each \textit{DO} performs the denoising operation locally via $w^\ast_k=\widehat{w}^\ast_k-r_k,(k=0, \cdots, n)$,  and finally obtains the global true regression model weights.

\subsection{Security Analysis}
\label{security}
There are three goals introduced in section \ref{sec:system}.
The first goal is to ensure each \textit{DO} deriving effective model weights, and has been addressed.
In this subsection, we analyze two other goals regarding the security of our FCD scheme.
We conduct security analyses in view of the \textit{DO}, the \textit{CSP}, and the \textit{Evaluator}, under the semi-honest-but-curious threat model given in section \ref{security protect}.

\subsubsection{Protection of \textit{DO}'s Private Data}
\textit{DO}'s private data will not be leaked to other parties, i.e., \textit{CSP} and \textit{Evaluator}.
Firstly, under the architecture of our FCD scheme, \textit{DO}'s private data is not transmitted to other parties.
They only share some encrypted intermediate quantities to the \textit{CSP}, e.g., $E[q^{(l)}_k]$, $E[s^{(l)}_{k j}]$,$E[z^{(l)}_k]$, $E[\Delta r^{(l)}_k]$.
Since the \textit{CSP} doesn't possess the sk, they have no access to these quantities.
Secondly, both \textit{CSP} and \textit{Evaluator} can derive the plaintext of aggregations of some intermediate quantities (with noise), i.e., $Q_k$, $S_{k j}$, ${Z}_k$, and $\Delta R_k$.
But it is easy to see that they couldn't infer any private information or data from these quantities.
To this end, \textit{DO}'s private data is effectively protected under our scheme.

\subsubsection{Protection on the Model Weights}
\label{protection}
Another essential security goal of our scheme is to guarantee accurate model weight private to both \textit{CSP} and \textit{Evaluator}.
From \textit{Theorem} \ref{the1} we can see that the \textit{Evaluator} can only derive optimal model weight $\widehat{w}^\ast_k = w^\ast_k + r_k$, which is inaccurate for each dimension $k(k=0,\cdots,n)$. 
The gap between actual optimal weight and inferred weight is $r_k$.
For \textit{CSP}, we give the following theorem.

\begin{theorem}\label{the2}
	Given $Q_k$, $S^{\prime}_{k j}$, ${Z}_k$ and $\Delta R_k$, suppose $w^\ast_k$ and $\widetilde{w}^\ast_k$ are the optimal actual weight and the inferred weight by \textit{CSP}. Let $|1-\xi_{kj}|\geq \epsilon$ where $\epsilon>0$ is an error bound. We have 
	$$
	|\widetilde{w}^\ast_k- w^\ast_k|\ \geq\ \epsilon \sum_{j=0}^{n} |S_{kj} w_j / Z_k |.
	$$
\end{theorem}
\begin{proof}
	We take lasso regression as example.
	 According to Eq.\eqref{equ:sc17} we have $S^{\prime}_{kj} = S_{kj}\cdot \xi_{kj}$.
	 Combining the optimal weight illustrated by Eq.\eqref{equ:wk3} and Eq.\eqref{equ:pk2}, we have
	 $$
		\begin{aligned}
			|\widetilde{w}^\ast_k- w^\ast_k| &= \left|\frac{\sum_{j=0}^{n} (S_{kj}- S^{\prime}_{kj})w_j}{Z_k} \right|\\
			&=\left|\frac{\sum_{j=0}^{n} (1- \xi_{kj})S_{kj}w_j}{Z_k} \right|\\
			&\geq \epsilon \sum_{j=0}^{n} |S_{kj} w_j / Z_k |.
		\end{aligned}
	$$
\end{proof}
From \textit{Theorem} \ref{the2}, we can see that the lower bound of the difference  between the optimal actual weight $w^\ast_k$ and the inferred weight $\widetilde{w}^\ast_k$ is $\epsilon \sum_{j=0}^{n} |S_{kj} w_j / Z_k |$, for each dimension $k(k=0,\cdots,n)$.
As a result, the \textit{CSP} cannot obtain or infer the accurate global regression model weights.

\section{Experimental Settings}
\label{experiment}
In this section, we conduct experiments and analyses on our FCD scheme in terms of accuracy, perturbations, computation cost, and communication overhead, respectively.
 
\subsection{Experimental Settings}

\begin{table}[t]
\centering
\caption{Details of the UCI datasets}
\label{Tab:08}
\resizebox{\linewidth}{!}{ 
\begin{tabular}{c c c c c}
\toprule
Id&Dataset&\# Features&\# Samples&References\\
\midrule
1&DD&10&442&~\cite{misc_diabetes_34}\\
2&BHD&13&506&~\cite{Dua:2019}\\
3&AD&8&4177&~\cite{misc_abalone_1}\\
\bottomrule
\end{tabular}
}
\end{table}

\textbf{Datasets.} Three public datasets from the UCI repository ~\cite{Dua:2019} and a synthetic dataset are used in this paper to test the performance of our FCD scheme.

\begin{itemize}
  \item UCI public datasets. We adopt the Boston House dataset (BHD), the Abalone dataset (AD), and the Diabetes dataset (DD) from the UCI machine learning database.
  The UCI datasets are used to test the accuracy of linear regressions under our FCD scheme, compared with the centralized methods.
  The details of UCI datasets are shown in Table \ref{Tab:08}.
  \item Synthetic dataset. We generate a synthetic dataset with $n$ features and $m$ samples. Both the inputs and the outputs of generated samples follow the standard multivariate normal distribution $\mathcal{N}(\textbf{0},\textbf{1})$.
  We vary the number of features and samples to test the computation cost of linear regressions under the FCD schemes.  
\end{itemize}

\noindent\textbf{Baseline.} 
To test the effectiveness of our scheme, we perform three kinds of regressions under FCD, i.e., linear, ridge, and lasso regression, denoted by FCD-Linear, FCD-Ridge, and FCD-Lasso, respectively.
For comparison, we also perform these regressions under the \textit{C}entralized \textit{T}raining \textit{L}inear \textit{R}egression scheme with gradient descent algorithm, and we use CTLR-Linear, CTLR-Ridge, and CTLR-Lasso to denote them.

To test the computation overheads, we compare our regression models in our FCD scheme with PrivFL, a federated learning scheme~\cite{DBLP:journals/corr/abs-2004-02264}.
PrivFL is an existing federated learning scheme based on a single cloud model.
However, PrivFL trains regression models using the gradient descent method, and thus 
only supports linear and ridge regressions.
Since our scheme support lasso regression, we only use PrivFL-Linear and PrivFL-Ridge to denote the regressions trained under the PrivFL scheme, respectively.

\noindent\textbf{Other settings.} We ran our experiments using Python 3.7 on a PC equipped with a 2.30-GHz Intel Core i7-10510U CPU, 8-GB of RAM, and Windows 10 system.

In the paillier encryption system, the key space is in the integer field based on $\mathbb{Z}_{N}$.
Considering that the data before encryption is in float64 units, it is necessary to perform encoding and decoding operations.
Then it can execute the encryption and decryption algorithm for floating-point numbers.
Suppose $a$ is the data to be encrypted in float64, we perform the operation $A=a \times c\ (c={10}^6)$ and encode $a$ to an integer, and then execute the encryption algorithm.
Meanwhile, we make the positive floating-point number encoding range $[0,N/2]$ and the negative floating-point number range $[N/2,n)$ ($N$ is denoted as $N=p\cdot q$ in \ref{paillier}).
Similarly, after decrypting to plaintext $A$, it is necessary to set $a=A/c$ to decode $A$ into plaintext form as a floating-point number.
In our experiments, we set the key length of the paillier encryption algorithm to 1024 bits.

\begin{table}
\centering
\caption{The MAE of regressions under FCD and CLTR schemes}
\label{Tab:01}

\begin{tabular}{p{23pt}p{23pt}p{23pt}p{23pt}p{23pt}p{23pt}p{23pt}}
\toprule
Dataset&FCD-Linear&CTLR-Linear&FCD-Ridge&CTLR-Ridge&FCD-Lasso&CTLR-Lasso\\
\midrule 
DD&0.49038&0.49083&0.57667&0.57660&0.49866&0.49866\\
BHD&0.37830&0.37797&0.33148&0.34904&0.38046&0.38046\\
AD&0.59056&0.59038&0.44202&0.44247&0.59219&0.59219\\
\bottomrule
\end{tabular}

\end{table}

\begin{figure*}[t] 
	\centering  
	\vspace{-0.35cm} 
	\subfigure[Linear regression]{
		\label{fig:outlier3}
		\includegraphics[width=0.30\linewidth]{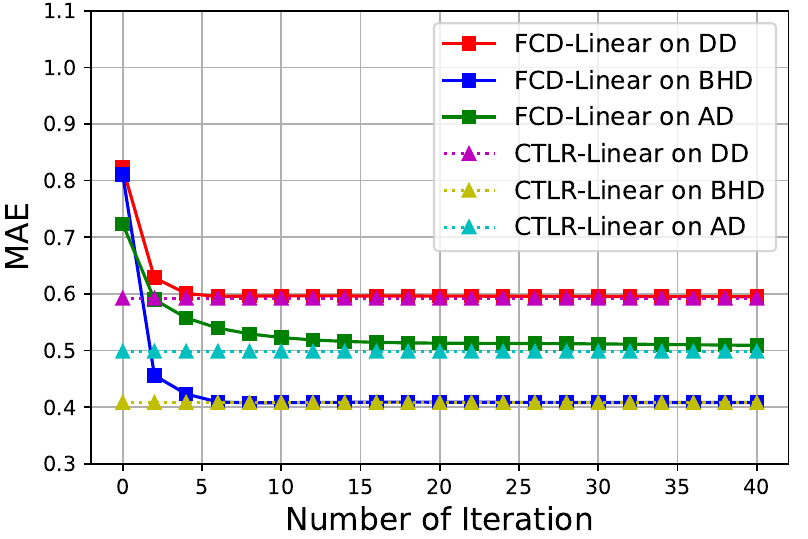}}
	\subfigure[Ridge regression]{
		\label{fig:outlier4}
		\includegraphics[width=0.30\linewidth]{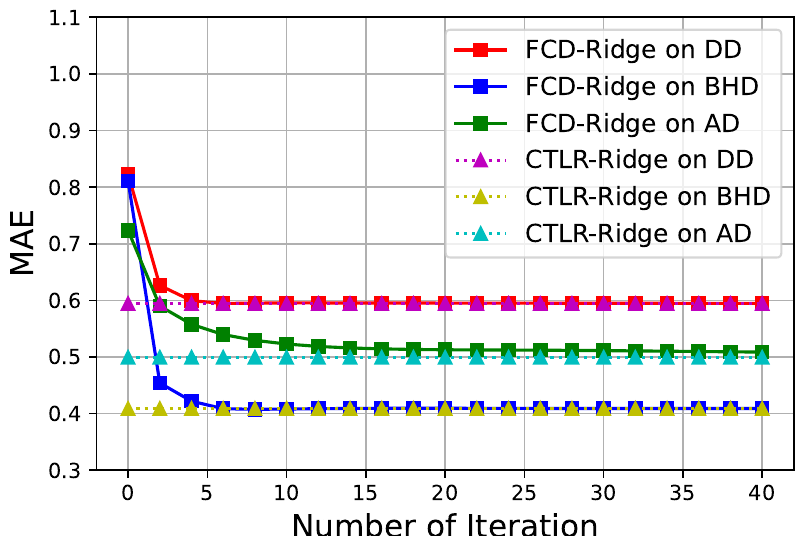}}
	\subfigure[Lasso regression]{
		\label{fig:outlier5}
		\includegraphics[width=0.30\linewidth]{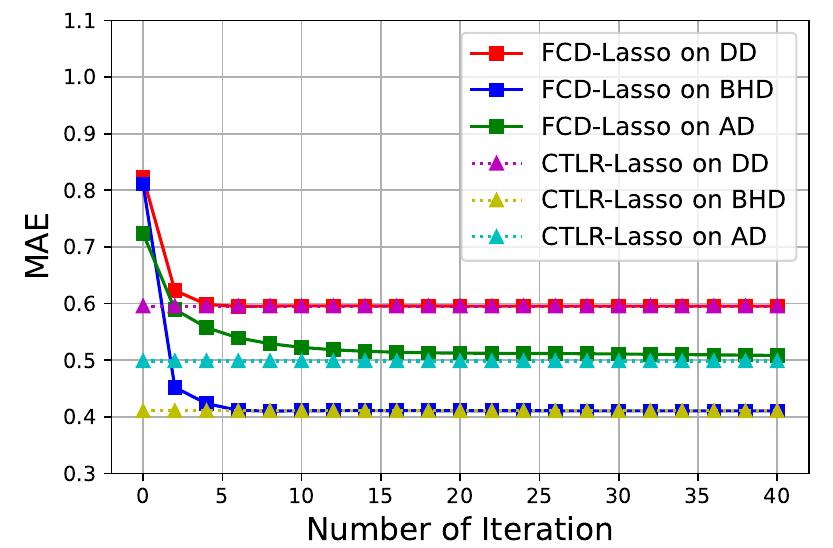}}
	\caption{The MAE of regression models with different number of iteration}
	\label{mae}
\end{figure*}

\begin{figure*}[t] 
	\centering  
	\vspace{-0.35cm} 
	\subfigure[Linear regression]{
		\label{fig:outlier6}
		\includegraphics[width=0.3\linewidth]{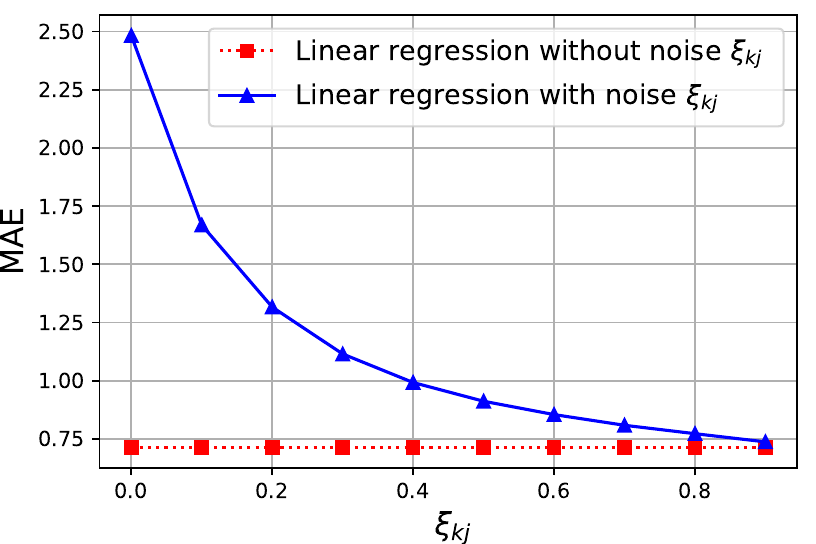}}
	\subfigure[Ridge regression]{
		\label{fig:outlier7}
		\includegraphics[width=0.3\linewidth]{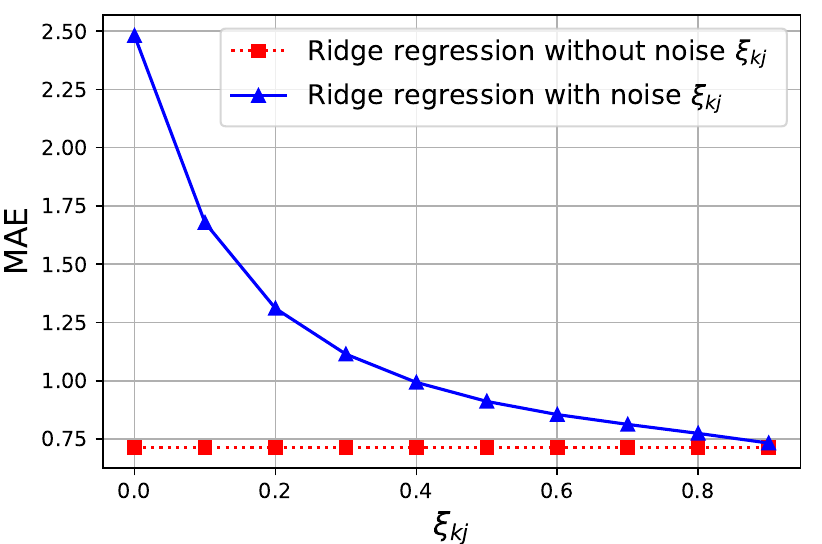}}
	\subfigure[Lasso regression]{
		\label{fig:outlier8}
		\includegraphics[width=0.3\linewidth]{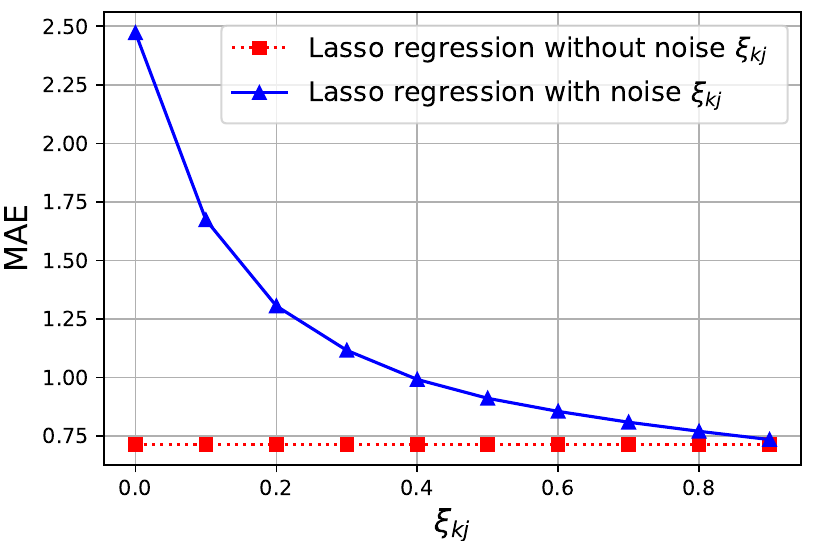}}
	\caption{The impact of MAE  by perturbation ${\xi}_{kj}<1$}
	\label{perturbation-csp}
\end{figure*}

\subsection{Accuracy Evaluation}
To evaluate the accuracy of regressions under our FCD scheme, we use the Mean Absolute Error (MAE) as the measure of model accuracy. 
We first fix the number of iterations $l=1000$ for the three regression models. 
Since ridge and lasso regression contain penalty terms, after several parameter searches, we set their penalty coefficients $\lambda=5$. 

Note that we randomly take 20\% samples from each dataset as the testing dataset. 
The remaining samples are randomly distributed equally to each \textit{DO} for training the global regression model.
The obtained results are shown in Table \ref{Tab:01}.
From the table, we can see that the MAE of the regressions under our distributed FCD scheme is very close to that of the centralized regression models.
Specifically, the MAEs of FCD-Lasso and CTLR-Lasso are exactly the same.
It means that our FCD is very effective with high regression accuracy.

To test the convergence of regression models in FCD, we compare the accuracy of three regression models for different number of iterations.
We record the MAE values of all regression models under both FCD and CTLR schemes on three UCI datasets for every two iterations, as shown in Fig. \ref{mae}.
We can see that for the first several iterations, the MAE values of FCD models are larger than CTLR models.
It is because that for our FCD scheme, data is distributed on multiple \textit{DO}s.
But after ten iterations, the MAE of FCD models is close to CTLR models until convergence.
For at most 20 iterations, our FCD regression models can converge and complete the training.

\begin{figure*}[t]
	\centering  
	\vspace{-0.35cm}
	\subfigure[Linear Regression]{
		\label{fig:outlier13}
		\includegraphics[width=0.30\linewidth]{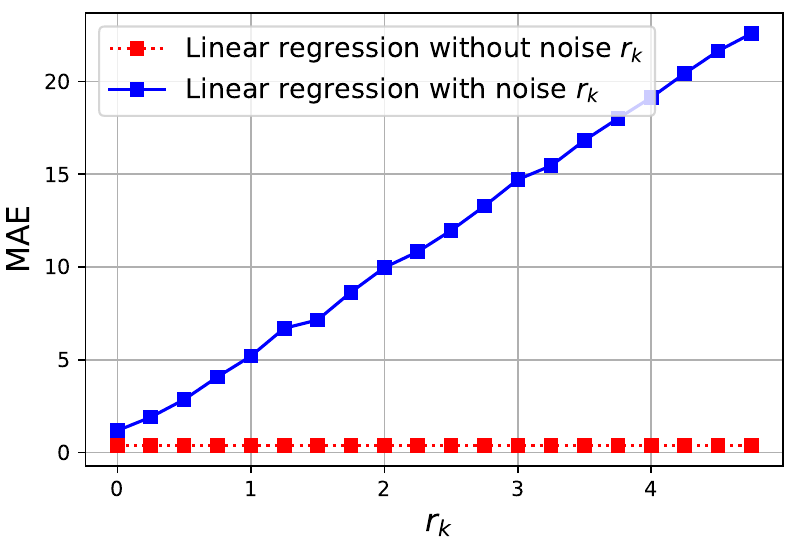}}
	\subfigure[Ridge Regression]{
		\label{fig:outlier14}
		\includegraphics[width=0.30\linewidth]{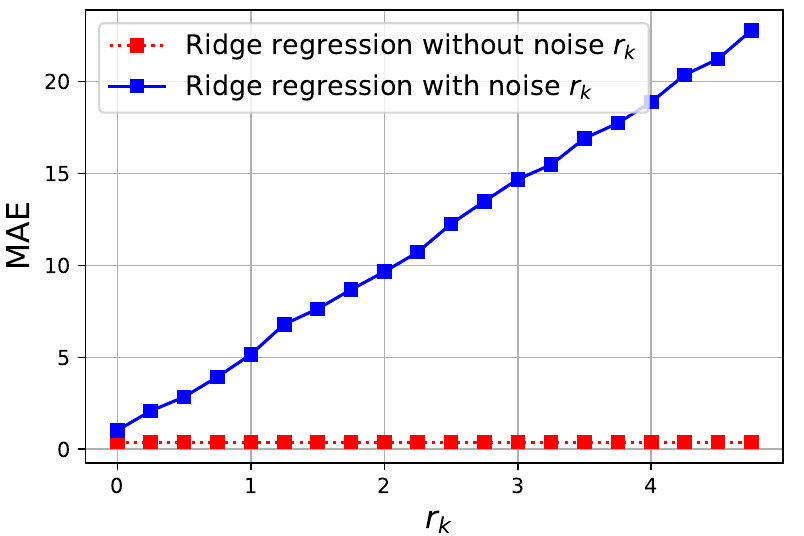}}
	\subfigure[Lasso Regression]{
		\label{fig:outlier15}
		\includegraphics[width=0.30\linewidth]{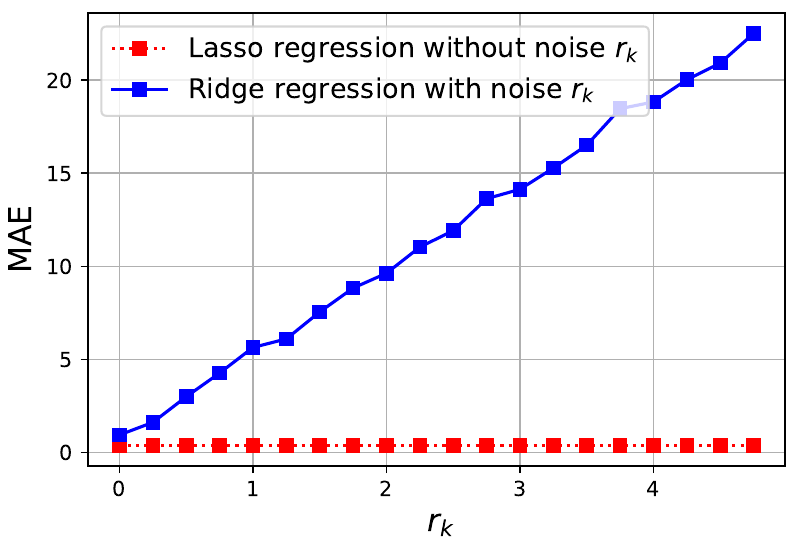}}
	\caption{The impact of MAE by perturbation $r_k$}
	\label{perturbation-rk}
\end{figure*}

\subsection{Perturbation Analysis}
\label{noise exp}
In this subsection, we focus on exploring the corelation between perturbations and the regression MAE values. 
Accoding to our results, we give the suggested values for perturbation numbers.
We take the number of iterations $l=50$, and AD as the test datasets.
\subsubsection{Perturbation analysis on $\xi_{kj}$}
$\xi_{kj}$ is one of the vital noises, which are added in $S^{\prime}_{kj}$ by the \textit{Evaluator} to prevent the \textit{CSP} from inferring the global model weights.
We investigate the impact of the value of $\xi_{kj}$ on the MAE values.
Since we adopted the paillier homomorphic scalar multiplication operations, $\xi_{kj} =1$ means an unsuccessful perturbation.
When we set $\xi_{kj} <1$, we show the MAE results in Fig. \ref{perturbation-csp}.
The red curve represents the model without any noise, and the blue curve denotes the model with perturbations. 
We can see that when $\xi_{kj}$ is close to 0, the MAE value goes higher, which means the model is well protected.
As the ${\xi}_{k j}$ is close to 1, the protection performance for the model becomes weaker.
\begin{table}
	\centering
	\caption{The MAE with  perturbation $\xi_{kj}\geq1$}
	\label{Tab:03}
	\begin{tabular}{c c c c} 
	\toprule 
	$\xi_{k j}$&FCD-Linear&FCD-Ridge&FCD-Lasso\\
	\midrule 
	1&0.7132&0.7135&0.7137\\
	1.02&1.56282E+13&5.31133E+16&1.54837E+13\\
	1.04&1.38258E+29&1.17751E+29&4.21769E+28\\
	1.06&2.09276E+38&1.09463E+43&7.26065E+39\\
	1.08&1.59273E+52&9.66648E+57&8.1003E+51\\
	1.10&1.99883E+70&3.92258E+67&2.45967E+71\\
	\bottomrule
	\end{tabular}
\end{table}

We also take 0.02 as the step size to conduct our experiments for ${\xi}_{kj}$ in $[1,1.1]$. We use a table to show the MAE results in Table \ref{Tab:03}.
We can see that the MAE increases vastly as ${\xi}_{k j}$ increases slightly.
To this end, we suggest the value of the perturbation ${\xi}_{k j}$ to $(0,0.2]$ or $[1.02,1.1]$.

\subsubsection{Perturbation analysis on $r_k$}
Perturbation $r_k$ is another significant noise added to the initial weights and intermediate quantities.
Without $r_k$, the \textit{Evaluator} can easily infer model weights.
We ask how the value of $r_k$ affects the MAE measure.
We take $r_{k} \in[0,5]$ in steps of 0.25, while fixing the number of iterations $l = 50$.
We illustrate the results in Fig. \ref{perturbation-rk}.
In the figure, the red curve represents the regression model without any noise, and the blue curve denotes the model with perturbation $r_k$.

From the figure, we can find that as $r_k$ increases, the MAE values of the perturbated model also increase, and the increasement is nearly linear to $r_k$. 
Thus, to protect the model and prevent the \textit{Evaluator} from trying to infer any information about the global regression model, we suggest the value of $r_k$ to be $[2, \infty)$.

\begin{table*}[t]
	\centering
	\caption{The computation cost of FCD and PrivFL schemes}
	\label{Tab:02}
 \resizebox{\textwidth}{!}{
\begin{tabular}{ c c c } 
	\toprule
	&FCD-Linear&PrivFL-Linear\\
	\midrule
	\textit{DO}&$2M\cdot(3n+n^2)\cdot\left(t_{exp}+t_{mul}\right)+(m\cdot(2n+n^{2})+n^{2})$ & $t\cdot\left((m+n+2m\cdot n)\cdot t_{mul}+2\cdot(m+n+m\cdot n+1)\cdot t_{exp}\right)$\\
	\textit{Evaluator}\&\textit{CSP}&$(3n+n^2)\cdot\left(t_{inv}+(M+2)\cdot t_{mul}+t_{exp}\right)+n^2\cdot t_{exp}+t\cdot(n^{2}+n)$&$ t\cdot\left((n+2M+1)\cdot t_{mul}+(2n+M+2)\cdot t_{exp}+M\cdot t_{{inv}}\right)$\\
	\bottomrule
	\end{tabular}
 }
	
\end{table*}

\subsection{Computation Cost}
In this section, we perform computation cost analysis through experiments. 
We analyze the computation cost theoretically and compare it with an existing federated regression scheme PrivFL ~\cite{DBLP:journals/corr/abs-2004-02264}. 
\subsubsection{Theoretical analysis}
We mainly consider the complex paillier computation operations in this paper as previous work ~\cite{isci/WangZLZL21} for the sake of simplicity.
Let $t_{mul}$, $t_{exp}$ and $t_{inv}$ denote the computation cost of modular multiplication, modular exponentiation, and modular inverse operations, respectively.
Note that $m$ denotes the number of samples, $n$ denotes the number of features, $M$ denotes the number of \textit{DO}s, and $t$ denotes the number of iterations.

According to our FCD scheme, \textit{DO} needs to compute and encrypt intermediate quantities firstly, so the computation cost of \textit{DO} is $2M\cdot(3n+n^2)\cdot\left(t_{exp}+t_{mul}\right)+(m\cdot(2n+n^{2})+n^{2})$. 
In the data aggregation and computation phase, \textit{Evaluator} needs to add noise to the aggregated data, and \textit{CSP} needs to decrypt the intermediate quantities.
The computation cost in \textit{Evaluator} and \textit{CSP} is: $M\cdot(3n+n^2)\cdot t_{mul}+n^2\cdot t_{exp}+t\cdot (n^2+n)$ and $(3n+n^2)\cdot\left(t_{exp}+2t_{mul}+t_{inv}\right)$ , respectively.
For linear and ridge regression of our FCD scheme, the total computation cost of \textit{Evaluator} and \textit{CSP} is $(3n+n^2)\cdot\left(t_{inv}+(M+2)\cdot t_{mul}+t_{exp}\right)+n^2\cdot t_{exp}+t\cdot (n^2+n)$.

For the lasso regression, since we use a secure multiparty computation algorithm to compare $P_k$ and $\pm \lambda/2$, an additional computation cost is required as $6n\cdot t\cdot t_{mul}+4n\cdot t\cdot t_{exp}+2n\cdot t\cdot\left(t_{exp}+2t_{mul}+t_{inv}\right)$. 
Thus, the computation cost in \textit{CSP} and \textit{Evaluator} of lasso regression is $\left(3n+n^2\right)\cdot\left(t_{inv}+\left(M+2\right)\cdot t_{mul}+t_{exp}\right)+6n\cdot t\cdot t_{mul} +4n\cdot t\cdot t_{exp}+n^2\cdot t_{exp}+2n\cdot t\cdot\left(t_{exp}+2t_{mul}+t_{inv}\right)+t\cdot (n^2+n)$.

We compare the computation cost between FCD and PrivFL schemes in Table \ref{Tab:02}. 
For the PrivFL, \textit{DO} calculates the global gradients encrypted in the cloud  during each iteration. 
Therefore, the computation cost of \textit{DO} and the cloud in PrivFL tends to increase linearly as the number of iterations and the number of samples increase.
However, for our FCD scheme, \textit{DO} only needs to encrypt the intermediate quantities and upload them to \textit{Evaluator} for model training. 
Thus, encryption operation is required on each \textit{DO} only once.
Our scheme significantly reduces the computation cost.

\subsubsection{Experiments on Synthetic dataset}

In this paper, we explore factors affecting the computation cost of our regression model, i.e., the number of features $n$, the number of samples $m$, the number of \textit{DO}s $M$, and the number of iterations $l$. 
We perform experiments on the synthetic dataset. 
Since the linear and ridge regression models have the same cost, we omit the ridge regression and only test the linear and lasso regression.

\noindent\textbf{Number of features.}
We set $m=1000,M=5,l=50$, and change the number of features $n$ in steps of 5. 
The experimental results are shown in Fig. \ref{features}.
We can see that the computation cost at \textit{DO}, \textit{CSP}, and \textit{Evaluator} increases by $O(n^2)$ as the feature $n$ increases.
The cost increase at \textit{DO} is more significant.
Simultaneously, the computation cost of lasso regression at \textit{CSP} and \textit{Evaluator} is much higher due to the secure multiparty computation.

\noindent\textbf{Number of samples.}
We set $n=10, M=5, l=50$, and change the number the samples $m$ in steps of 500 for the experiment.
The experimental results are shown in Fig. \ref{samples}.
It is clear that for both linear and lasso regression, the computation cost on \textit{Evaluator} and \textit{CSP} keeps stable when the number of samples increases.
The computation cost on \textit{DO} increases by $O(m)$ as the number of samples $m$ increases.
Therefore, we believe that our scheme maintains a low computational cost after  \textit{DO}s finish their encryption.
\begin{figure}[tbp] 
	\centering 
	\vspace{-0.35cm} 
	\subfigure[Linear regression]{
		\label{fig:features1}
		\includegraphics[width=0.47\linewidth]{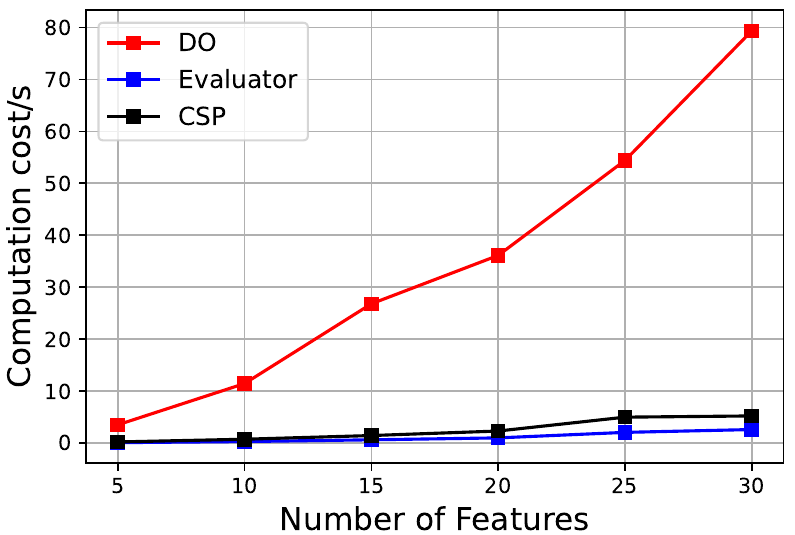}}
	\subfigure[Lasso regression]{
		\label{fig:features2}
		\includegraphics[width=0.48\linewidth]{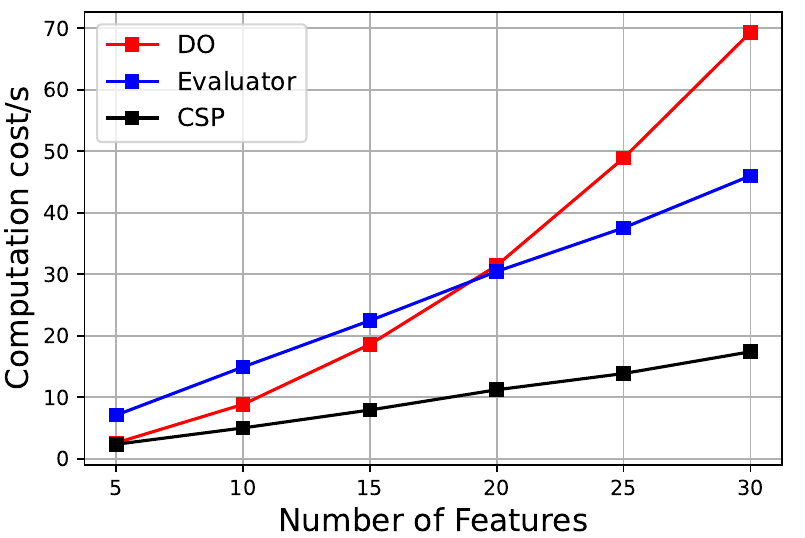}}
	\caption{Computation cost in different number of features}
	\label{features}
\end{figure}

\begin{figure}[tbp] 
	\centering 
	\subfigure[Linear regression]{
		\label{fig:Samples1}
		\includegraphics[width=0.47\linewidth]{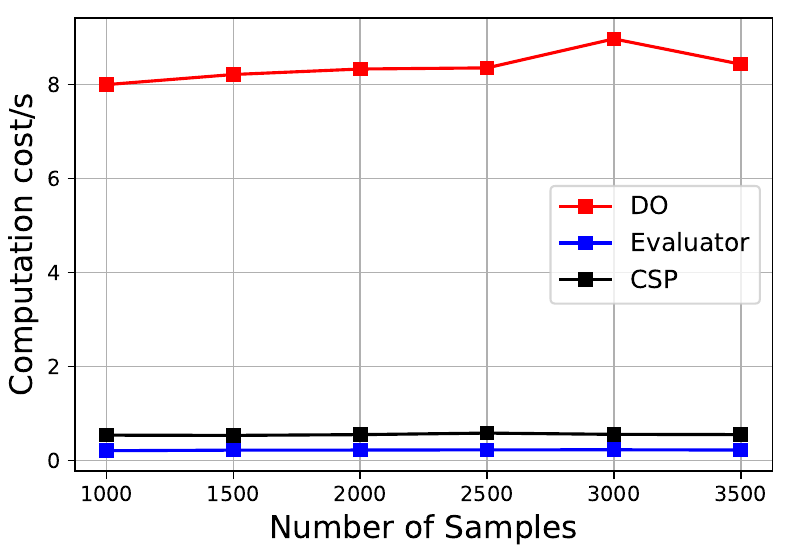}}
	\subfigure[Lasso regression]{
		\label{fig:Samples2}
		\includegraphics[width=0.48\linewidth]{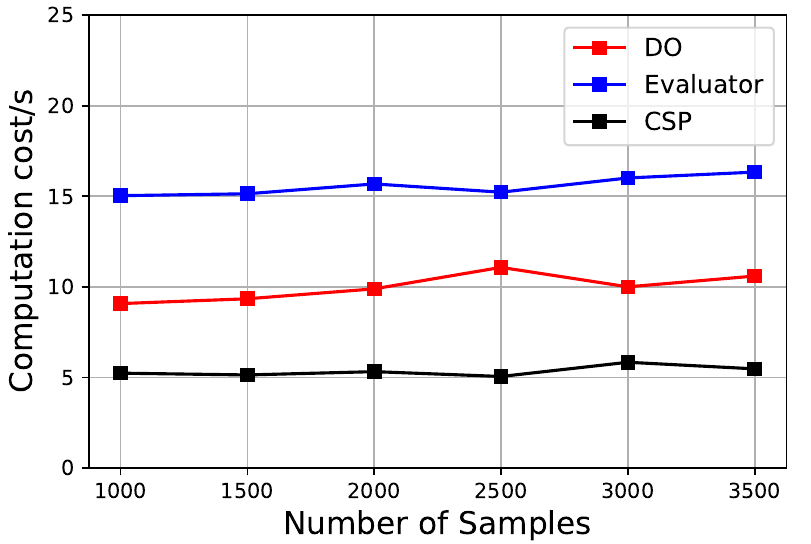}}
	\caption{Computation cost in different number of samples}
	\label{samples}
\end{figure}

\begin{figure}[tbp] 
	\centering 
	\vspace{-0.35cm} 
	\subfigtopskip=2pt
	\subfigbottomskip=2pt 
	\subfigcapskip=-5pt 
	\subfigure[Linear regression]{
		\label{fig:DO1}
		\includegraphics[width=0.45\linewidth]{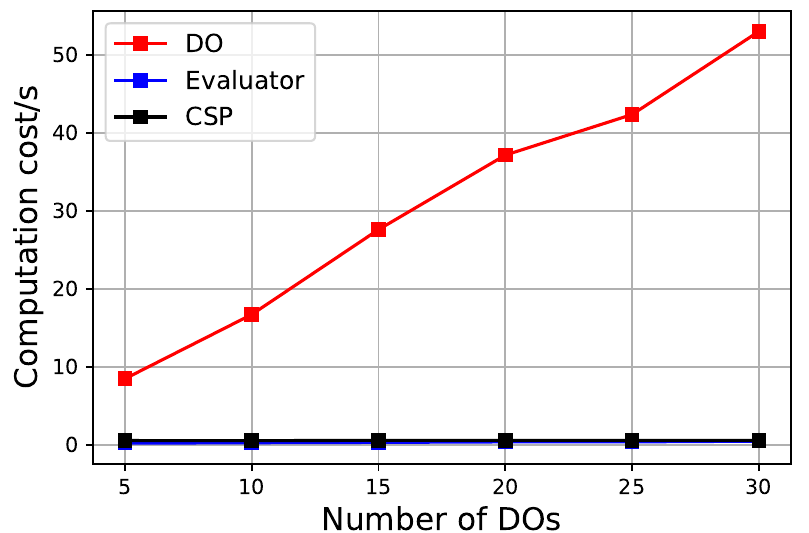}}
	\subfigure[Lasso regression]{
		\label{fig:DO2}
		\includegraphics[width=0.45\linewidth]{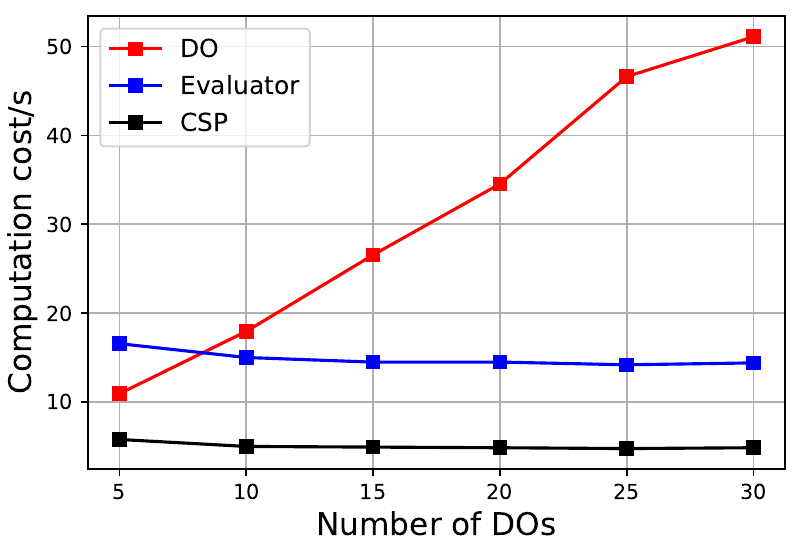}}
	\caption{Computation cost in different number of \textit{DO}s}
	\label{DOs}
\end{figure}

\noindent\textbf{Number of \textit{DO}s.}
We set $n=10, m=1000, t=50$, and change the number the \textit{DO}s $M$ in steps of 5 for the experiment.
The experimental results are shown in Fig. \ref{DOs}. 
We can see that for both linear and lasso regression, the computation cost in \textit{DO} increases linearly with the number of \textit{DO}s, which is also consistent with our theoretical analysis.
Meanwhile, the computation cost in the \textit{CSP} and \textit{Evaluator} keeps stable.

\noindent\textbf{Number of iterations.}
We set $n=10, m=1000, M=5$, and change the number of the iterations $t$ in steps of 20 for the experiment.
The computation cost results in \textit{Evaluator} and \textit{CSP} are shown in Fig. \ref{Iterations}.
Since only \textit{Evaluator} and \textit{CSP} are involved in the training phase, Fig. \ref{Iterations} doesn't display computation cost in \textit{DO}.
We can find that for linear regression, the computation cost of \textit{Evaluator} increases linearly as the number of iterations increases, while \textit{CSP} maintains a more stable computation time.
The computation cost of \textit{Evaluator} rises slowly with the number of iterations increases, while \textit{CSP} remains stable. 
For lasso regression, the computation cost of \textit{Evaluator} and \textit{CSP} grows more significantly because of the multiparty security comparison algorithm, but it is still acceptable.

\begin{figure}[tbp] 
	\centering  
	\vspace{-0.35cm} 
	\subfigtopskip=2pt 
	\subfigbottomskip=2pt 
	\subfigcapskip=-5pt 
	\subfigure[Linear regression]{
		\label{fig:Iterations1}
		\includegraphics[width=0.45\linewidth]{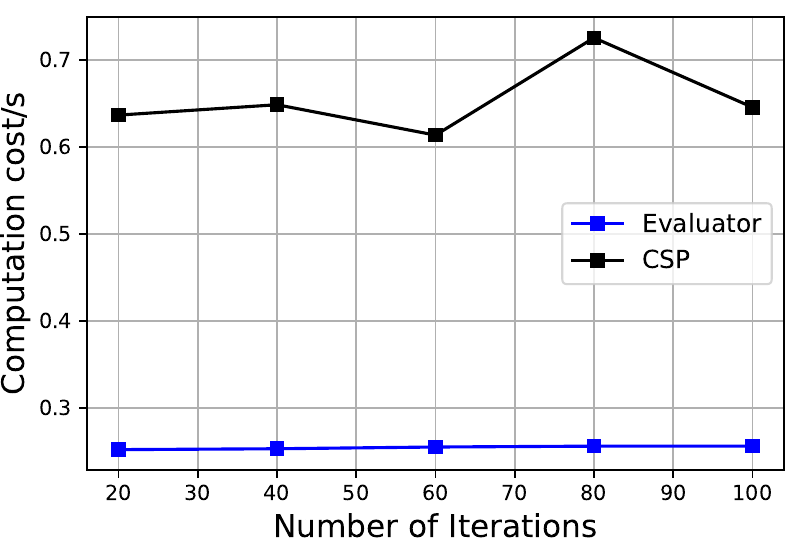}}
	\subfigure[Lasso regression]{
		\label{fig:Iterations2}
		\includegraphics[width=0.45\linewidth]{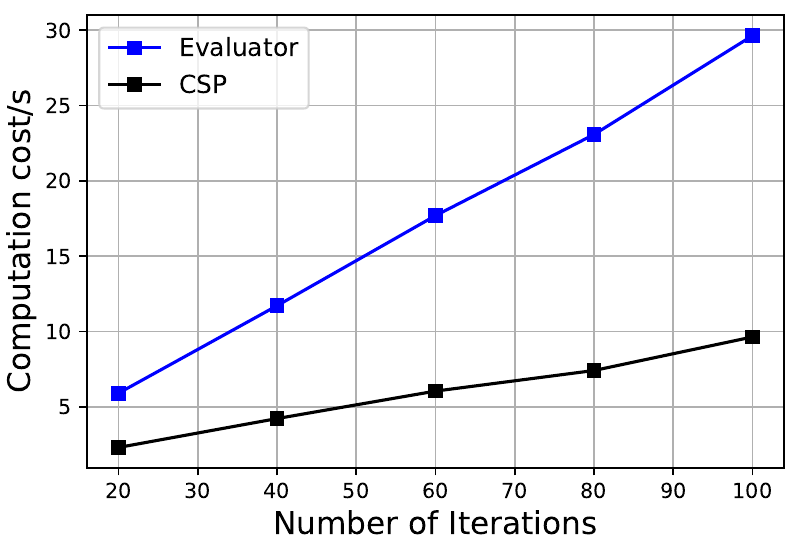}}
	\caption{Computation cost in different number of iterations}
	\label{Iterations}
\end{figure}

\subsection{Communication overhead}
In this subsection, we perform a theoretical analysis of our FCD scheme on communication overhead and compare it with PrivFL.

For the FCD-ridge and FCD-linear regression model, we assume that the key length is $K$ bits, which results in a communication overhead of $2K$ bits per encrypted element.
In the local computation phase, \textit{DO} needs to encrypt the local computation data before uploading it to \textit{Evaluator}, so the communication overhead of the scheme in this phase is $2M\cdot(n^2+3n)\cdot K$.
In the data aggregation phase, \textit{Evaluator} needs to send the aggregated encrypted intermediate quantities to the \textit{CSP} for decryption.
The communication overhead between the \textit{Evaluator} and \textit{CSP} is $(n^2+3n)\cdot2K$.
Note that the decryption data type is float64, accounting for 64 bits, so the communication overhead of the whole training process is $ \left(n^2+3n\right)\cdot\left(M+1\right)\cdot 2K+64\left(n^2+3n\right)$.
For lasso regression, the multiparty secure comparison algorithm is an additional communication overhead.
As the length of an integer is 32 bits, the total communication overhead of lasso regression is $\left(n^2+3n\right)\cdot\left(M+1\right)\cdot 2K+64\left(n^2+3n\right)+(4K+32)\cdot n\cdot t$.

\begin{table}[tb]
\centering
\caption{The communication overhead of FCD and PrivFL schemes}
\label{Tab:09}
\resizebox{\linewidth}{!}{ 
\begin{tabular}{ccc} 
\toprule
&FCD&PrivFL\\
\midrule
Linear/Ridge&$ \left(n^2+3n\right)\cdot\left(M+1\right)\cdot 2K+64\left(n^2+3n\right)$& $2M\cdot\left(n+2\right)\cdot t\cdot K$\\
Lasso&$\left(n^2+3n\right)\cdot\left(M+1\right)\cdot 2K+64\left(n^2+3n\right)+(4K+32)\cdot n\cdot t$& \textbackslash\\
\bottomrule
\end{tabular}
}
\end{table}

We summarize the communication overhead of the FCD and PrivFL schemes in Table .\ref{Tab:09}. 
From the table, we can find that our FCD scheme outperforms PrivFL in  communication overhead if more iterations are needed for convergence.
Firstly, for linear and ridge regressions, our scheme requires communication between \textit{DO} and \textit{Evaluator} only once, and also communication once between \textit{Evaluator} and \textit{CSP}.
But the PrivFL scheme needs to execute data transfer under ciphertext during each iteration. 
As a result, the communication overhead of the PrivFL increases linearly as the number of iterations increases.
Secondly, our scheme support lasso regression, which couldn't be performed under the PrivFL scheme.

\section{Conclusion}
\label{conlusion}
In this paper, we propose a federated scheme for multiparty linear regression, which introduces homomorphic encryption and perturbation techniques to prevent possible data and model weights leakage during the collaborative training phases.
The most significant advantage of our proposed scheme is a coordinate descent solution for secure multiparty lasso regression, which hasn't been well addressed as yet, since the gradient descent-based method could not deal with $\mathcal{L}_1$ normalization term.
We also give security theoretical analyses on our scheme.
The experimental results have demonstrated that our FCD scheme has achieved as competitive results as centralized regression methods, in terms of prediction performance, computational cost, and communication overhead.

\section*{Acknowledgments}

This work was supported in part by the National Natural Science Foundation of China (Grant No. 61802124, 61931019), and the Fundamental Research Funds for the Central Universities (Grant No. 2021MS089).

\bibliographystyle{cas-model2-names}

\bibliography{cas-refs}

\end{document}